\documentclass{article}

\usepackage{show}
\usepackage[utf8]{inputenc} 
\usepackage[T1]{fontenc}    
\usepackage{natbib}  
\setcitestyle{numbers,square,comma}
\usepackage{hyperref}
\hypersetup{
  colorlinks=true,  
  linkcolor=purple,    
  citecolor=blue,     
  urlcolor=blue,    
}   
         
\usepackage{booktabs}       
\usepackage{amsfonts}       
\usepackage{nicefrac}       
\usepackage{microtype}      
\usepackage{enumerate}
\usepackage{enumitem}
\usepackage{doi}
\usepackage{todonotes}
\usepackage{xcolor}
\usepackage{tcolorbox} 
\usepackage{adjustbox}   
\usepackage{xcolor}   
\usepackage{geometry}
\definecolor{brightblue}{RGB}{0, 122, 255}
\usepackage{fontawesome}

\usepackage{url}
\usepackage{graphicx}
\usepackage{subcaption}
\usepackage{algorithm}
\usepackage{algpseudocode}

\usepackage{listings}
\usepackage{colortbl}
\definecolor{textcolor}{RGB}{0, 0, 0}

\lstdefinestyle{PythonStyle}{
    language=Python,
    basicstyle=\ttfamily\small\color{textcolor},
}

\usepackage{booktabs} 
\usepackage{multirow}
\usepackage{array}

\usepackage{amsmath, amssymb}

\usepackage{amsthm}
\newtheorem{lemma}{Lemma}
\newtheorem{remark}{Remark}
\newtheorem{theorem}{Theorem}

\newtheorem{proposition}{Proposition}

\title{
  The Debate on RLVR Reasoning Capability Boundary: Shrinkage, Expansion, or Both? A Two-Stage Dynamic View
}

\author{
Xinhao Yao$^{1,2}$\thanks{Work done during an internship at Ant Group.} \quad
Lu Yu$^{2}$ \quad
Xiaolin Hu$^{3}$ \quad
Fengwei Teng$^{1}$ \quad
Qing Cui$^{2}$ \quad
Jun Zhou$^{2}$ \quad
Yong Liu$^{1}$\thanks{Corresponding author: \texttt{liuyonggsai@ruc.edu.cn}.} \\[1ex]
$^{1}$Renmin University of China\qquad
$^{2}$Ant Group \qquad
$^{3}$Xiamen University
}

\renewcommand{\shorttitle}{A preprint}

\begin{document}
\maketitle
\begin{center}
\begin{abstractbox}
The ongoing debate on whether reinforcement learning with verifiable rewards (RLVR) expands or shrinks the reasoning capabilities of large language models (LLMs) remains unresolved. Some studies contend that RLVR mainly improves sampling efficiency but at the expense of diversity and exploratory capacity, resulting in \textit{capability boundary shrinkage}. In contrast, others demonstrate that prolonged training can lead to the emergence of novel reasoning strategies, suggesting \textit{capability boundary expansion}. To reconcile these contradictory findings, we theoretically and empirically show that \textit{\textbf{both}} perspectives are partially valid—each aligning with a separate phase in an inherent two-stage probability mass dynamic: (1) \textit{Exploitation} stage: initially, the model primarily samples explored high-reward and low-reward tokens, while rarely selecting the potentially optimal token. Positive advantage estimates increase the probability of high-reward tokens and decrease those of low-reward tokens, yet the optimal token’s probability remains largely unchanged during this stage. (2) \textit{Exploration} stage: as training advances, the growth rate of previously acquired high-reward tokens slows as their probabilities approach saturation. When a potentially optimal token—now receiving positive advantage estimates—is occasionally sampled, its probability increases, while those of the originally high-reward tokens decrease. This dynamic suggests that over-exploitation during the exploitation stage may lead to capability boundary shrinkage, whereas prolonged training into the exploration stage can promote an expansion of the reasoning capability boundary. Building upon our insights, we revisit the potential of only using relative negative gradients for prolonging training, providing a theoretical and empirical foundation for the development of more advanced reasoning capabilities.
\vspace{1em}

\raisebox{-0.35ex}{\includegraphics[width=3.8cm]{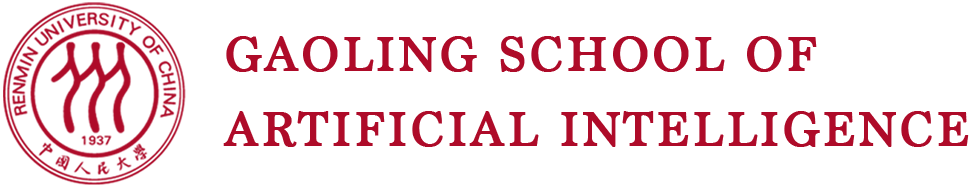}}%
\hfill
\raisebox{-0.35ex}{\includegraphics[width=1.6cm]{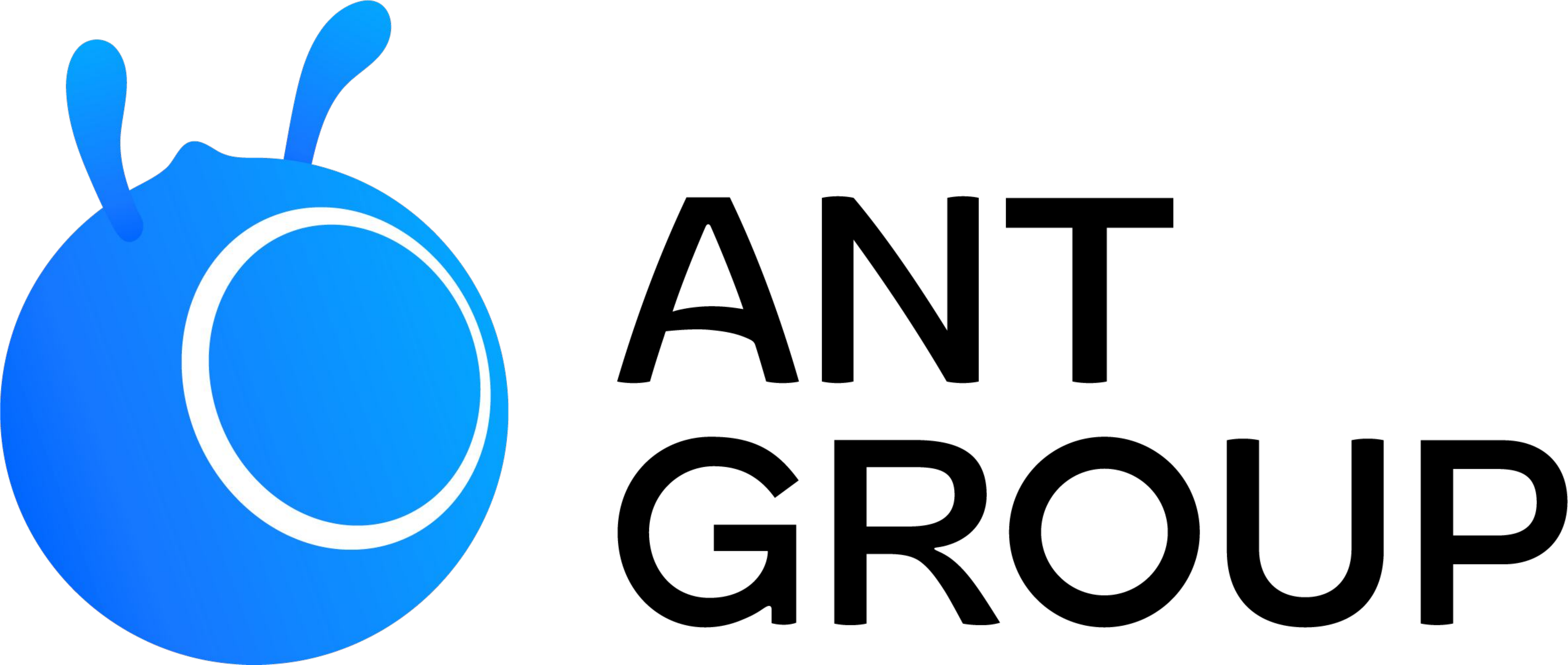}}%
\end{abstractbox}
\end{center}

\section{Introduction}
Reinforcement learning with verifiable rewards (RLVR) has become a key paradigm for substantially enhancing the reasoning abilities of large language models (LLMs), as exemplified by advanced models such as OpenAI’s O1 and O3 \citep{jaech2024openai,openai2025o3} and DeepSeek-R1 \citep{guo2025deepseek}. By optimizing pre-trained or chain-of-thought (CoT) \citep{cot} fine-tuned models through verifiable reward signals, RLVR enables LLMs to excel in complex logical tasks such as mathematics \citep{luo2025deepscaler,yue2025vapo,zeng2025simplerl} and programming \citep{liu2025code,luo2025deepcoder}.

Despite empirical successes, a fundamental question is still hotly debated: \textit{does RLVR genuinely expand the reasoning capabilities of base models beyond their original boundaries?} Current evidence is sharply divided. (1) One line of research \citep{yue2025does,zhao2025echo,dang2025assessing,he2025rewarding,ma2025learning,shah2025rethinking,gandhi2025cognitive} argues for \textbf{capability boundary shrinkage}, contending that while RLVR improves sampling efficiency, it fails to produce genuinely novel reasoning strategies and may even induce a progressive narrowing of reasoning capabilities during training. Empirical evidence from \citet{yue2025does} shows that although RLVR-trained models perform better under small-$k$ sampling (e.g., $k=1$), base models achieve higher Pass@$k$ when $k$ is large. Similarly, \citet{cui2025entropy} document a sharp entropy collapse during training, resulting in overly deterministic behavior \citep{zhang2025policy} and reduced exploratory effectiveness. (2) In contrast, another body of work \citep{liu2025prorl,wen2025reinforcement,liu2025scaling,wu2025invisible,yuan2025llms,wang2025emergent,sun2025deltacodedoesrlunlock} provides evidence supporting \textbf{capability boundary expansion}. \citet{liu2025prorl} attribute previous evidence of capability boundary shrinkage to the premature termination of RL training, which disrupts learning before novel reasoning capabilities can fully develop. Through prolonged training, they further demonstrate that RLVR can explore and populate
new regions of solution space over time. Meanwhile, \citet{wu2025invisible} experimentally show that RLVR can occasionally expand empirical support, producing novel correct solutions beyond the original reach of the base model. 

\textit{The debate between these two lines of evidence centers on empirical results; however, the underlying mechanisms responsible for these contradictory findings remain unclear}.  To elucidate the mechanisms, we focus on the evolution of the policy model’s probability mass distribution—termed the \textit{probability mass dynamics}. As a conceptual starting point, consider that the search tree \citep{yue2025does,zhu2025reasoning,hao2024training} for any given prompt is built through iterative sampling from the policy. This tree grows exponentially at a rate of $\mathcal{O}(V^T)$, where $V$ denotes the vocabulary space (token set) size and $T$ the maximum generation length. Crucially, policy updates can be viewed as a dynamic reallocation of probability mass across the search tree, thereby shaping the reasoning capability boundary.

Through an integrated theoretical and empirical analysis (Section \ref{main:sec_pmdte}), \textbf{we demonstrate that both lines of evidence hold validity to some extent}—each corresponding to a distinct stage within a two-stage dynamic of probability mass. Specifically, since the logit for token $v$ is directly tied to its policy probability—a larger (smaller) logit results in a higher (lower) probability—we analyze the policy gradient of the training objective and derive a bidirectional update rule for the logits (i.e., the pre-Softmax values; Lemma \ref{main:lemma1}). According to this rule, updates to the logits depend on both the advantage estimate $\hat{A}$ and the current policy distribution $\pi$. Under practical optimization settings such as GRPO \citep{deepseekmath} (where multiple responses are sampled per prompt), Theorem \ref{main:proposition2} establishes that the expected logit update for token $v$ is proportional to $\pi(v)\left[(1-\pi(v))\hat{A}(v) - \sum_{u \neq v} \pi(u)\hat{A}(u)\right]$.

From this view, the overall dynamic appears to unfold in two distinct stages. (1) \textbf{Exploitation stage}: initially, the model predominantly samples the already-explored high-reward token and the low-reward token, while the potentially optimal token is selected only infrequently. Driven by positive advantage estimates, the probability of the high-reward token increases, whereas that of the low-reward token decreases. However, the probability of the potentially optimal token remains largely unchanged throughout this stage. This behavior suggests that \textit{over-exploitation during this stage may result in a shrinkage of the capability boundary}. (2) \textbf{Exploration stage}: as training progresses, the growth rate of the high-reward token previously explored slows as its probability approaches near saturation ($1 - \pi \rightarrow 0$). When the potentially optimal token—now associated with positive advantage estimates—is occasionally sampled, its probability increases, while that of the formerly high-reward token declines. A key characteristic of this dynamic is the transition of the relative negative sample: from the initially low-reward token to the high-reward token. This implies that with prolonged training, gradient updates can be progressively redirected toward tokens with low initial probability but high potential, once high-probability tokens have stabilized, ultimately \textit{expanding the reasoning capability boundary}. We illustrate these theoretical insights with a toy example (Section \ref{main:toy}).

Building on our theoretical and experimental insights, a direct way to expand the reasoning capability boundary and mitigate shrinkage is to prolong training while concentrating policy probability updates exclusively on optimizing relative negative samples (denoted -N, Section \ref{main:subsetup}) throughout the learning process. Empirical investigations (Section \ref{main:-n}) of our strategy—implemented in widely adopted algorithms (e.g. GRPO, GSPO \citep{gspo}) on benchmark datasets and open-source LLMs verify that GRPO-N (GSPO-N) achieves competitive and stable performance improvements while largely preserving the base model's diversity, demonstrating the potential for prolonged training. Notably, analysis of the training process reveals instances where incorrect code is initially generated but is later refined and corrected through iterative reflection. Unlike GRPO, which reinforces the entire trajectory—including error-prone steps—GRPO-N effectively prevents such reinforcement.

$\diamond$ \textbf{Main contributions}. Briefly, this study unveils the underlying mechanisms responsible for the heated debate (boundary shrinkage or expansion) in RLVR from both theoretical and practical perspectives. We emphasize the essential role of fine-grained probability mass allocation and establish a theoretical and empirical basis for understanding the impact of RLVR on reasoning capabilities.

\subsection{More Related Works}
Broadly speaking, our work builds upon lines of research in reinforcement learning for LLM reasoning, LLM learning dynamics, and gradient analysis in preference optimization. A comprehensive review of related work is included in Appendix \ref{app:works} due to page constraints.
\newpage
\section{Preliminaries and Background}\label{main:sec_pre}
In this section, we describe the core components of our study by reviewing some basic notations. 

\textbf{RLVR}. Reinforcement learning with verifiable rewards (RLVR) is a paradigm for improving models on tasks with objectively verifiable outcomes. In this formulation, an autoregressive language model is treated as a policy $\pi_{\theta}$ (parameter $\theta$). For a given query $\mathbf{x}$ from a prompt set $\mathcal{D}$, the probability of generating a response $\mathbf{y}$ is defined as $\pi_{\theta}(\mathbf{y}\mid \mathbf{x})=\prod_{t=1}^{|\mathbf{y}|} \pi_{\theta}(y_t \mid \mathbf{x}, \mathbf{y}_{<t})$. A deterministic reward function $r$ assigns a scalar value indicating the correctness of the full response $\mathbf{y}$ to the prompt $\mathbf{x}$. Each token in $\mathbf{y}$ receives the same reward (1 only if the final answer is correct, and 0 otherwise). The objective is to minimize the loss: $\mathcal{L}_{\text{RLVR}}(\theta) = -\mathbb{E}_{\mathbf{x} \sim \mathcal{D},\ \mathbf{y} \sim \pi_{\theta}(\cdot\mid\mathbf{x})} \left[ r(\mathbf{x}, \mathbf{y}) \right]$, where $r(\mathbf{x}, \mathbf{y}) \in [0,1]$.

\textbf{A unified framework for policy gradient optimization}. Building on the work of \citep{deepseekmath,liu2025intro,swamy2025all}, we consider a unified objective $\mathcal{J}$ that establishes connections among various optimization methods:
\begin{equation}\label{eq:main1}
     \mathcal{J}_{\text{RLVR}}(\theta) = \mathbb{E}_{\mathbf{x} \sim \mathcal{D},\ \mathbf{y} \sim \pi_{\theta_{\text{old}}}(\cdot\mid\mathbf{x})}
\left[
\frac{1}{|\mathbf{y}|} \sum_{t=1}^{|\mathbf{y}|} \min\left(
w_t(\theta) \hat{A}_t,
\text{clip}\big(w_t(\theta), 1 - \epsilon, 1 + \epsilon\big) \hat{A}_t
\right)
\right],
\end{equation}
where $\epsilon$ is a clipping hyperparameter, $\text{clip}(\cdot)$ is the clipping operation, and the the importance ratio of the token $y_t$ is defined as $w_t(\theta)=\frac{\pi_{\theta}(y_t \mid \mathbf{x}, \mathbf{y}_{<t})}{\pi_{\theta_{\text{old}}}(y_t \mid \mathbf{x}, \mathbf{y}_{<t})}$ (the current policy $\pi_{\theta}$ and the old policy $\pi_{\theta_{\text{old}}}$). $\hat{A}_t$ is the advantage of current token and is implemented differently across optimization methods:

$\bullet$ \textbf{PPO} (Proximal Policy Optimization \citep{schulman2017proximalpolicyoptimizationalgorithms,ouyang2022traininglanguagemodelsfollow}). $\hat{A}_t$ is computed by applying Generalized Advantage Estimation (GAE) \citep{schulman2018highdimensionalcontinuouscontrolusing}, based on the value model. This incurs considerable computational and memory overhead, and its effectiveness critically depends on the reliability of its value estimation.

$\bullet$ \textbf{GRPO} (Group Relative Policy Optimization \citep{deepseekmath}). To reduce variance, GRPO and its variants (e.g., DAPO \citep{yu2025dapoopensourcellmreinforcement} \& Dr.GRPO \citep{liu2025understandingr1zeroliketrainingcritical}) eliminate reliance on a value model by using Monte Carlo estimates to compute the relative advantage across a group of responses $\{\mathbf{y}_i\}_{i=1}^G \sim \pi_{\theta_{\text{old}}}$ to the same query (where $G$ is the group size and all token in $\mathbf{y}_{i}$ share the same relative advantage): 
$$w_{i,t}(\theta) = \frac{\pi_{\theta}(y_{i,t} \mid \mathbf{x}, \mathbf{y}_{i,<t})}{\pi_{\theta_{\text{old}}}(y_{i,t} \mid \mathbf{x}, \mathbf{y}_{i,<t})},\quad \hat{A}_{i,t} = \hat{A}_{i} = \frac{r(\mathbf{x}, \mathbf{y}_i) - \text{mean}\left(\{r(\mathbf{x}, \mathbf{y}_i)\}_{i=1}^G\right)}{\text{std}\left(\{r(\mathbf{x}, \mathbf{y}_i)\}_{i=1}^G\right)}.$$

$\bullet$ \textbf{GSPO} (Group Sequence Policy Optimization \citep{gspo}). Given that the token-level importance ratio $w_{i,t}$ in GRPO does not align with sequence-level rewards, GSPO introduces a sequence-level importance ratio $w_i$ based on sequence likelihood \citep{zheng-etal-2023-click}:
$$w_i(\theta) = \left( \frac{\pi_\theta (\mathbf{y}_i \mid \mathbf{x})}{\pi_{\theta_{\text{old}}} (\mathbf{y}_i \mid \mathbf{x})} \right)^{\frac{1}{|\mathbf{y}_i|}} = \exp \left( \frac{1}{|\mathbf{y}_i|} \sum_{t=1}^{|\mathbf{y}_i|} \log \frac{\pi_\theta (y_{i,t} \mid \mathbf{x}, \mathbf{y}_{i,<t})}{\pi_{\theta_{\text{old}}} (y_{i,t} \mid \mathbf{x}, \mathbf{y}_{i,<t})} \right).$$

To better understand the model's learning dynamics under this binary outcome reward setting, we omit the regularization components\footnote{Regularization components are widely regarded as mechanisms for ensuring training stability. Moreover, studies \citep{hu2025openreasonerzeroopensourceapproach,gpg} indicate that omitting them does not impair performance when others are properly tuned.} (e.g., KL term \& clipping operation). That is, the policy gradient $\nabla_{\theta}\mathcal{J}_{\text{RLVR}}(\theta)$ can be simplified to $\mathbb{E} \left[\frac{1}{|\mathbf{y}|} \sum_{t=1}^{|\mathbf{y}|} w_t(\theta)\hat{A}_{t} \nabla_{\theta} \log \pi_{\theta}(y_{t} \mid \mathbf{x}, \mathbf{y}_{<t}) \right]$ with respect to $\theta$. Specifically, taking GRPO as an example (Appendix \ref{app:proof_eq2} for derivation):
\begin{align}\label{eq:main2}
 \nabla_{\theta}\mathcal{J}_{\text{GRPO}}(\theta) \ = \mathbb{E}_{\mathbf{x},\{\mathbf{y}_i\}_{i=1}^G } \left[ \frac{1}{G} \sum_{i=1}^{G} \frac{1}{|\mathbf{y}_i|} \sum_{t=1}^{|\mathbf{y}_i|} 
\underbrace{w_{i,t}(\theta)\hat{A}_{i,t}}_{\text{coefficient }} 
\nabla_{\theta} \log \pi_{\theta}(y_{i,t} \mid \mathbf{x}, \mathbf{y}_{i,<t}) \right].
\end{align}
\begin{remark}
    Intuitively, if we set $\hat{A}_{i,t}=1$ and $w_{i,t}=1$ while all $\mathbf{y}_i$ are correct responses, then Eq.(\ref{eq:main1}) essentially performs maximum likelihood estimation, i.e., supervised fine-tuning (SFT). Furthermore, Eq.(\ref{eq:main2}) indicates that the scalar $w_{i,t}\hat{A}_{i,t}$ can be interpreted as a weighting coefficient that adjusts the log-likelihood term. This implies that RLVR methods can be viewed as a form of reweighted SFT, where correct responses and incorrect responses contribute positive and negative gradients, respectively \citep{deepseekmath,gpg,zhu2025surprising,deng2025effectnegativegradientgroup,abdolmaleki2025learning,chen2025spectral}. When $\hat{A}_{i,t}$ is calculated from a comparison of average rewards across groups (e.g., GRPO), the resulting gradient is named the relative policy gradient.
\end{remark}

\section{Probability Mass Dynamics}\label{main:sec_pmdte}
As described above, we begin by considering a standard task that involves generating a reasoning sequence. In this setting, the model learns a policy $\pi_{\theta}(\mathbf{y} \mid \mathbf{x}) = \prod_{t=1}^{T} \pi_{\theta}(y_t \mid \mathbf{x}, \mathbf{y}_{<t}) \in \mathbb{R}^{V\times T}$ to map an input $\mathbf{x}$ to a sequence of predictions $\mathbf{y} = \{y_1, \dots, y_T\}$, where $\mathbf{y} \in \mathcal{V}^T$, $\mathcal{V}$ is the vocabulary space of size $V$, and $T$ denotes the maximum generation length. Conceptually, the reasoning process can be regarded as a tree search \citep{yue2025does,zhu2025reasoning,hao2024training}. A search tree is constructed for a given problem by iteratively sampling from the policy model. This process leads to exponential growth in the tree size, $\mathcal{O}(V^T)$, reflecting an open-ended and combinatorially infinite reasoning space \citep{secondhalf}. 

Crucially, policy updates can be viewed as dynamically reallocating probability mass over the search tree, thereby shaping the boundary of reasoning capability. Here, we specifically focus on the evolution of the policy model’s probability distribution—referred to as probability mass dynamics.
 
$\diamond$ Learning dynamics offer critical insights into the key challenges and counterintuitive behaviors of deep learning \citep{renlearning}, with early explanations pointing to network ``stiffness" \citep{fort2019stiffness} or ``local elasticity" \citep{he2020local,deng2021toward}. 
To track the evolution of the probability distribution, we monitor the logits $\mathbf{z}^{\theta} \in \mathbb{R}^{V\times T}$ and the log probabilities $\log \pi_{\theta}(\mathbf{y}\mid \mathbf{x})$, where $\pi_{\theta}$ is derived from $\mathbf{z}^{\theta}$ via a column-wise Softmax $\pi_{\theta}(\cdot \mid \mathbf{x}, \mathbf{y}_{<t}) = \text{Softmax}(\mathbf{z}^{\theta}(\mathbf{x}, \mathbf{y}_{<t}))$. The \textbf{probability mass dynamics} are then defined as:
\begin{align}
   \Delta \mathbf{z}^{l}(\mathbf{x}) &\triangleq \mathbf{z}^{\theta^{l+1}}(\mathbf{x})-\mathbf{z}^{\theta^{l}}(\mathbf{x}), \\
   \Delta \log \pi^l(\mathbf{y}\mid \mathbf{x}) &\triangleq \log \pi_{\theta^{l+1}}(\mathbf{y}\mid \mathbf{x}) - \log \pi_{\theta^{l}}(\mathbf{y}\mid \mathbf{x}),
\end{align}
where the model's parameter $\theta$ is updated from step $l$ to $l+1$ by performing one policy gradient update on the sample data $(\mathbf{x},\mathbf{y})$. For simplicity, we primarily analyze the case where $T=1$ (i.e., $\mathbf{y} \in \mathcal{V}$), meaning $\Delta \mathbf{z}^{l} \in \mathbb{R}^{V \times 1}$ and its dimension aligns with the size of the model's vocabulary. Notably, a larger (smaller) logit results in a higher (lower) probability. For $T>1$, the updates can be computed separately; therefore, we can calculate the the distinct $T$ updates and stack them together.

\subsection{A Two-Stage Dynamic: Exploitation and Exploration}\label{main:subsec_pmd}
Given the monotonicity of the Softmax function, the main text focuses mainly on characterizing the changes of \textbf{z}. Analysis of the updates to $\log \pi_{\theta}$ with respect to $\theta$ is provided in the Appendix \ref{app:proof_p1}.

\begin{lemma}[Logits Update for Softmax Parameterization]\label{main:lemma1} Consider a policy parameterized by a Softmax function over logits $\mathbf{z}(\mathbf{x}):=\mathbf{z}=[z_1,\cdots,z_V]^T$, such that the probability of action (or token) $v$ is given by $\pi(v):=\pi(v\mid \mathbf{x})= \text{Softmax}(\mathbf{z})_v=\exp{(z_v)}/\sum^V_{v^{'}}\exp{(z_{v^{'}})}$. Reviewing Eq.(\ref{eq:main2}), \textbf{if the currently sampled action} is $v$, let the policy gradient estimate be $\nabla_{\mathbf{z}} \mathcal{J} \approx \hat{A}(v)\nabla_{\mathbf{z}} \log \pi(v)$. For a learning rate $\eta$, the update rule for the logits at time step $l$ is 
 (Appendix \ref{app:proof_lemma1} for derivation):
\begin{itemize}
    \item For the sampled action $v$:
        \[
        z_v^{l+1} \gets z_v^l + \eta \cdot \hat{A}(v) \cdot \big(1 - \pi^l(v)\big), \quad\Delta z_v^{l} = \eta \cdot \hat{A}(v) \cdot \big(1 - \pi^l(v)\big),
        \]
    \item For all other actions $u \neq v$:
        \[
        z_u^{l+1} \gets z_u^l + \eta \cdot \hat{A}(v) \cdot \big({- \pi^l(u)}\big), \quad\Delta z_u^{l}= \eta \cdot \hat{A}(v) \cdot \big({- \pi^l(u)}\big).
        \]
\end{itemize}
\end{lemma}
\begin{remark}[Bidirectional Update Rule]\label{main:remark2} The update to the logit $\mathbf{z}\in \mathbb{R}^{V\times1}$ depends on both the advantage estimate $\hat{A}$ and the current policy distribution $\pi$. Specifically, Let $v$ denote the currently sampled action. (1) when $\hat{A}(v) > 0$: $z_v$ increases by $\eta\hat{A}(v)(1-\pi(v))$ while $z_u$ ($u \neq v$) decreases by $\eta\hat{A}(v)\pi(u)$; (2) when $\hat{A}(v) < 0$: $z_v$ decreases by $\eta|\hat{A}(v)|(1-\pi(v))$ while $z_u$ ($u \neq v$) increases by $\eta|\hat{A}(v)|\pi(u)$. The normalization property of Softmax ensures that when $\hat{A}(v) > 0$, the update increases $\pi(v)$ while decreasing $\pi(u)$ for all $u \neq v$, including other advantageous actions. In contrast, when $\hat{A}(v) < 0$, the update increases the probabilities of other actions proportionally to their current policy values. The update may reallocate probability mass toward other potentially advantageous actions that were previously under-sampled.
\end{remark}
The practical update in group policy optimization (e.g. GRPO, DAPO, GSPO, REINFORCE++ \citep{hu2025reinforceefficientrlhfalgorithm}, GPG \citep{gpg}, GPO \citep{gpo}), which employs Monte Carlo sampling, arises from the collective effect of a group of responses, thus motivating our analysis of the \textbf{expected logits update}.
\begin{theorem}[The Expected Logits Update]\label{main:proposition2} Under the conditions stated in Lemma \ref{main:lemma1}, we assume\footnote{Without loss of generality, we approximate importance ratio $w\approx1$, as regularization components such as the KL penalty and clipping operation are applied in practical training.} that $\mathbf{x}  \sim \mathcal{D}$ is i.i.d. and $\{u_i\}_{i=1}^G$ are randomly sampled from $\pi(\cdot\mid \mathbf{x})$, the expected group relative policy gradient $\nabla_{\mathbf{z}} \mathcal{J} \in \mathbb{R}^{V\times1}$ is $\mathbb{E}_{\mathbf{x}\sim \mathcal{D},\{u_i\}_{i=1}^G \sim \pi(\cdot\mid \mathbf{x})} \left[\frac{1}{G}\sum_{i=1}^G\hat{A}(u_i) \nabla_{\mathbf{z}} \log \pi(u_i)\right]$. Then the expected logits update is (proof in Appendix \ref{app:proof_p2}):
\begin{align*}
    \mathbb{E}(\Delta z_v^l) =\eta\cdot\pi^l(v)\left[(1-\pi^l(v))\hat{A}(v)-\sum_{u\neq v}\pi^l(u)\hat{A}(u)\right].
\end{align*}
\end{theorem}
\begin{figure}[htb]
  \centering
    \begin{subfigure}[b]{0.245\textwidth}
        \centering
        \includegraphics[height=3.45cm]{   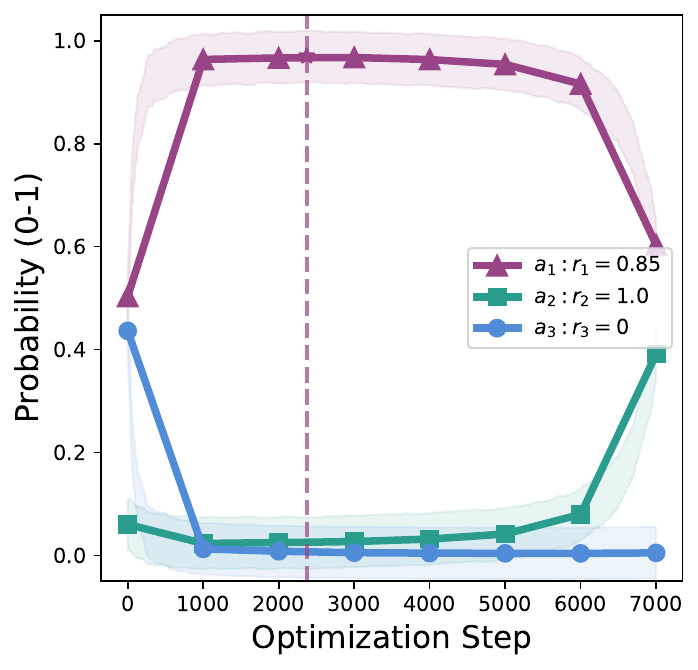}
    \end{subfigure}
    \begin{subfigure}[b]{0.245\textwidth}
        \centering
        \includegraphics[height=3.45cm]{   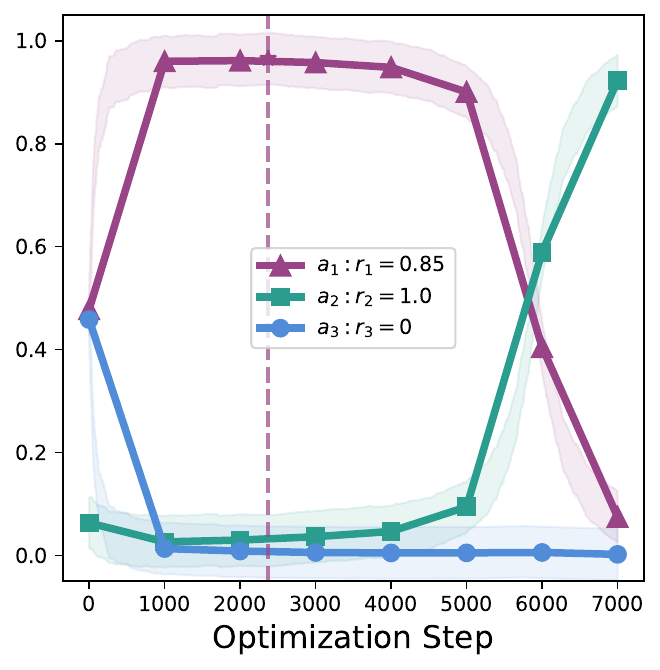}
    \end{subfigure}
    \begin{subfigure}[b]{0.245\textwidth}
        \centering
        \includegraphics[height=3.45cm]{   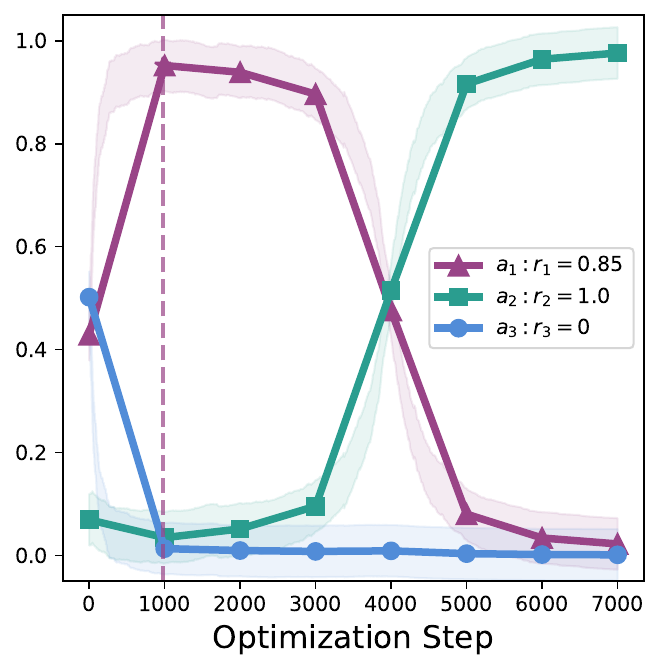}
    \end{subfigure}
    \begin{subfigure}[b]{0.245\textwidth}
        \centering
        \includegraphics[height=3.45cm]{   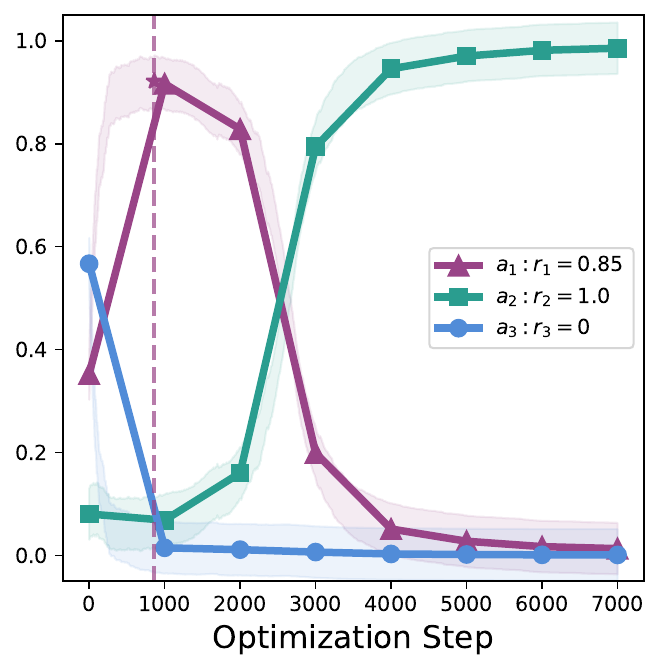}
    \end{subfigure} 
    \begin{subfigure}[b]{0.245\textwidth}
        \centering
        \includegraphics[height=3.45cm]{   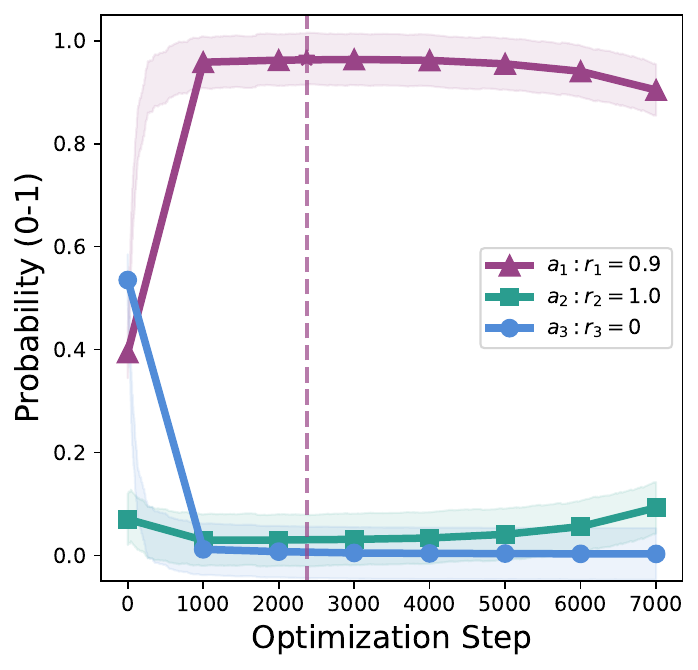}
    \end{subfigure}
    \begin{subfigure}[b]{0.245\textwidth}
        \centering
        \includegraphics[height=3.45cm]{   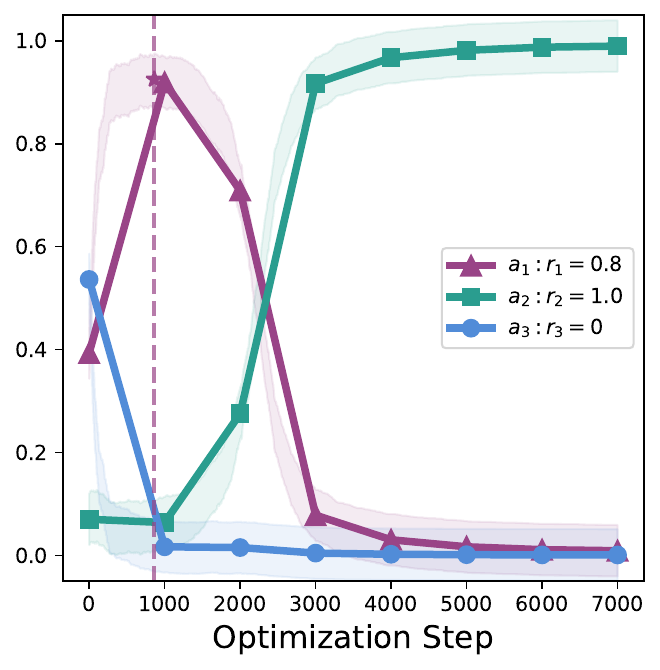}
    \end{subfigure}
    \begin{subfigure}[b]{0.245\textwidth}
        \centering
        \includegraphics[height=3.45cm]{   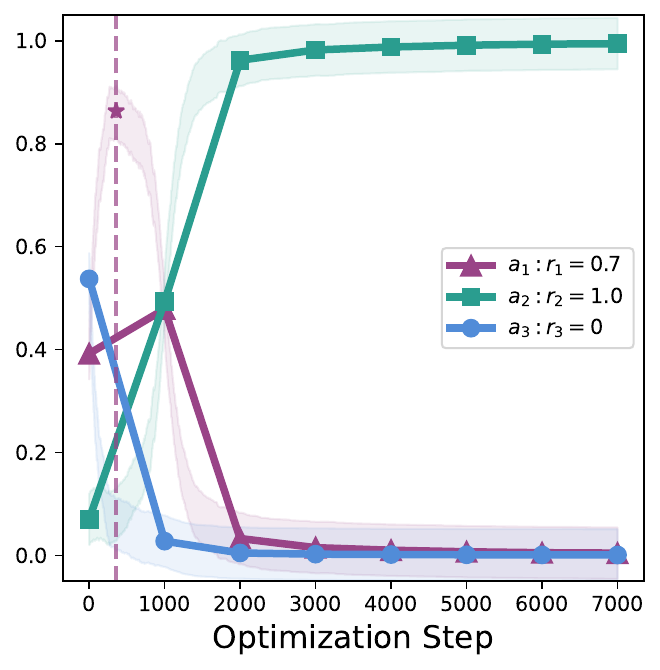}
    \end{subfigure}
    \begin{subfigure}[b]{0.245\textwidth}
        \centering
        \includegraphics[height=3.45cm]{   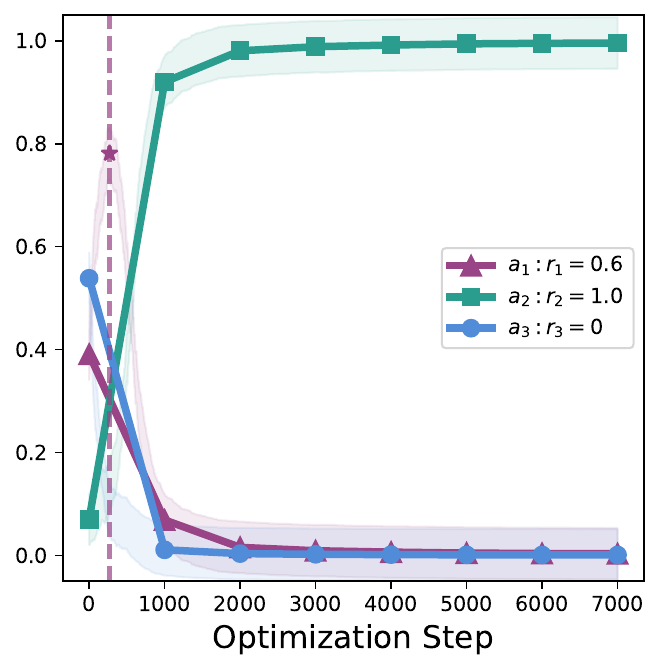}
    \end{subfigure} 
  \caption{The probability mass dynamics of policy optimization across varying action rewards $r$ and initial policy probabilities $\pi$.  Each sub-figure corresponds only to the indicated rewards and probabilities. The first row compares the impact of different initial policy probabilities under identical rewards, while the second row compares the effect of varying rewards given the same initial policy.}
  \label{main:figure1}
\end{figure}
\begin{remark}[A Two-Stage Dynamic of Exploitation and Exploration]\label{main:remark3} Theorem \ref{main:proposition2} establishes that the magnitude of the expected logit update $\Delta z_v^l$ is explicitly governed by $\pi^l(v)$. Although the Softmax function guarantees strictly positive probabilities for all actions, a significant number of these actions lie within the extreme tail of the probability distribution. As a result, under finite-sample training conditions, such actions exert negligible influence on parameter updates and are effectively omitted during optimization (that is, $\pi^l(v)\rightarrow 0$ leads to $\Delta z_v^l \rightarrow 0$). 
Interestingly, the overall dynamic appears to unfold in two distinct stages. (1) \textbf{Exploitation stage, corresponding to capability boundary shrinkage}: initially, the model mainly samples already-explored high-reward and low-reward tokens, rarely selecting the potentially optimal one. Driven by positive advantage estimates, the probability of the high-reward token increases while the low-reward token decreases. The potentially optimal token’s probability remains largely unchanged (or may even decrease, Remark \ref{main:remark2}), suggesting that over-exploitation in this stage may cause capability boundary shrinkage. (2) \textbf{Exploration stage, corresponding to capability boundary expansion}: As training continues, the growth of the previously dominant high-reward token slows as it approaches saturation ($1 - \pi \rightarrow 0$). When the potentially optimal token—now receiving positive advantage signals—is occasionally sampled, its probability rises, while that of the former high-reward token decreases. A key feature of this stage is the shift in relative negative sampling: from the initial low-reward token to the once high-reward token. This implies that through prolonged training, gradient updates can be shifted toward tokens with low initial probability but high reward, once high-probability tokens have stabilized. For instance, under the $\text{Pass@}k$ metric, raising the probability of at least one correct action above $1/k$ corresponds to an expansion of the reasoning capability boundary.
\end{remark}

\subsection{Demonstration with A Toy Example}\label{main:toy}
Next, to more clearly demonstrate the theoretically predicted two-stage dynamic, we validate the above analysis of probability mass dynamics using a simple toy setting, and subsequently review several widely adopted RLVR tricks and more than three actions case in Appendix \ref{app:discussion tricks}.

$\diamond$ \textbf{Starting with a toy setting}. To better track the probability mass dynamics, we analyze the scenario in a clean and simplified setting, assuming the entire action space consists of only three actions\footnote{In Section \ref{main:sec_pre}, the reward $r$ is sparse (1 or 0). However, due to the presence of factors such as the importance ratio $w$, distinctions arise among positive rewards ($wr$). For brevity of analysis, we ignore the standard deviation $\text{std}(\{r(a_j)\}_{j=1}^G)$ because it does not affect the sign (positive or negative nature) of $\hat{A}(a_i)$.}: $a_1$, with $r(a_1) > 0$, which has been explored; $a_2$, with $r(a_2) > 0$, which remains unexplored; and $a_3$, with $r(a_3) = 0$, which has been explored. Let the initial logits be denoted as $\mathbf{z}=[z(a_1),z(a_2),z(a_3)]^T$, and the policy as $\pi(a_i)=\exp{(z(a_i))}/\sum_{j\neq i}\exp{(z(a_j))}, \forall i\in[1,2,3]$. Here, we perform $G$ action samplings, estimate the relative advantages via $\hat{A}(a_i) = r(a_i) - \text{mean}(\{r(a_j)\}_{j=1}^G)$, and subsequently update the logits $\mathbf{z}$ using the policy gradient update rule given in Theorem \ref{main:proposition2}. As stated in Remark \ref{main:remark3}, we discuss the following scenario\footnote{Otherwise, we can proceed with normal optimization to increase the probability of the optimal action.}: RLVR reinforces high-probability yet suboptimal actions while overlooking potentially optimal correct actions that initially have low probability, thereby leading to over-exploitation behavior. That is, $r(a_1)<r(a_2)$, while initially $\pi(a_1)>\pi(a_2)$. For implementation details, see Algorithm \ref{alg1}.

\textbf{Results for demonstration}. We sample actions at each optimization step (with $G=2,\eta=0.1$) and analyze the probability mass dynamics (a larger $G$ leads to more stable optimization and does not affect our main findings and conclusions). Figure \ref{main:figure1} clearly illustrates the dynamics of the probability mass of policy optimization across different rewards $r$ \& initial policy probabilities $\pi$, which aligns with the theoretical analysis in Section \ref{main:subsec_pmd}. That is, $\mathbb{E}(\Delta z(a_i))=\eta\pi(a_i)\left[(1-\pi(a_i))\hat{A}(a_i)-\sum_{j\neq i}^3\pi(a_j)\hat{A}(a_j)\right]$. More specifically, the overall dynamics can be divided into two stages: (1) Initially, $\pi(a_1)$ and $\pi(a_3)$ are relatively large while $\pi(a_2)$ is comparatively small. Since actions $a_1$ and $a_3$ are predominantly sampled, and given that $\hat{A}(a_1)>0$ and $\hat{A}(a_3)<0$, $\pi(a_1)$ increases while $\pi(a_3)$ decreases. Meanwhile, $\pi(a_2)$ remains almost unchanged. (2) As $\pi(a_1)$ increases, the gradient term $1-\pi(a_1)$ gradually approaches zero, causing the growth of $\pi(a_1)$ to stabilize. If training continues beyond this point, when action $a_2$ is sampled with $\hat{A}(a_2)>0$ and $\hat{A}(a_1)<0$, $\pi(a_1)$ will decrease while $\pi(a_2)$ increases. Note that throughout the optimization process, the relative negative actions change (initially $a_3$ and later $a_1$).

\begin{remark}
From the two-stage dynamics, (1) it can be observed that although the relative policy gradient method does exhibit the phenomenon of capability boundary
shrinkage. However, prolonging the duration of the training may result in further gradient updates being applied to low-probability action sequences once the high-probability ones have reached convergence. This is precisely why the research represented by \citet{cui2025entropy} employs entropy control mechanisms to extend the duration of training. (2) More interestingly, the relative policy gradient may undergo changes during the training process: $\pi(a_1)$ first increases and then decreases. Therefore, simply using the momentum of policy gradients from the early stages of updates—as in methods like AAPO \citep{xiong2025aapo}—to enhance policy optimization is suboptimal. In contrast, approaches such as ProRL \citep{liu2025prorl,liu2025scaling} periodically \textbf{reset the reference policy and optimizer states} during training.
\end{remark}

\section{How to Prolong Training: Revisiting the Role of Relative Negative Gradients}\label{main:-n}
Thus far, we have established the imperative of avoiding over-sharpening in the policy distribution—which induces over-exploitation and entropy collapse—and of enabling sustained training. \citet{liu2025prorl} identify a fundamental limitation across existing studies \citep{yue2025does,dang2025assessing,zhao2025echo}: RL training is frequently terminated prematurely after only a few hundred steps, hindering the models' ability to fully explore and acquire novel reasoning capabilities. Their conclusions align closely with our findings. Therefore, enhancing training stability and facilitating extended training durations constitute promising directions for future research.

In group policy optimization (e.g. GRPO, GSPO), the policy $\pi_{\theta}(\mathbf{y}\mid \mathbf{x})=\prod_{t=1}^{|\mathbf{y}|} \pi_{\theta}(y_t \mid \mathbf{x}, \mathbf{y}_{<t})$ learned by the model is inherently complex. Returning to Section \ref{main:sec_pmdte}, policy updates can be interpreted as dynamically redistributing probability mass across the search tree, which has a size of $\mathcal{O}(V^T)$. To unlock the model's capacity for genuinely novel reasoning, we call for research into strategies that more effectively allocate probability mass. Based on the probability mass dynamics established in Lemma \ref{main:lemma1} and Theorem \ref{main:proposition2}, optimizing relative negative advantage actions implicitly increases the probability of other actions. A straightforward strategy is to allocate policy probability mass exclusively through relative negative gradients within the overall dynamics. In this part, we will revisit the role of only using relative negative gradients in prolonging training.

\subsection{Experimental Setup}\label{main:subsetup} We choose \texttt{Qwen2.5-Math-7B} \citep{yang2024qwen2} and \texttt{Llama-3.2-3B-Instruct} \citep{touvron2023llama}  as our base models for investigation, which align with our hardware resource. \textbf{For RLVR algorithms}, we evaluate the standard approach alongside a variant that employs exclusively relative negative gradients\footnote{The relative advantage $\hat{A}_{i,t}$ is computed over a group of responses. When $\hat{A}_{i,t}<0$, the term $w_{i,t}(\theta)\hat{A}_{i,t}\nabla_{\theta} \log \pi_{\theta}(y_{i,t} \mid \mathbf{x}, \mathbf{y}_{i,<t})$ is referred to as a relative negative gradient. See Appendix \ref{app:details_n} for the \textbf{details of relative negative gradients}.} (denoted as \textbf{-N}). This comparison includes widely-used methods such as GRPO \citep{guo2025deepseek,deepseekmath} and GSPO \citep{gspo}. Moreover, we use the verl framework \citep{sheng2025hybridflow} to train the models and the detailed hyperparameter settings of training and evaluation can be found in Appendix \ref{app:sub_settings}. \textbf{For the datasets}, we employ the training set of MATH \citep{hendrycks2021measuring}, which comprises 7,500 problems, for model training (with prompt batch size of 1,024). Performance is evaluated on widely-used reasoning benchmarks, (1) in-domain (ID) tasks: the test sets of MATH, AIME 2024, AIME 2025, and AMC 2023. (2) out-of-domain (OOD) tasks: ARC-c \citep{arc} (open-domain reasoning), MMLU-Pro \citep{mmlu} (academic reasoning).

Specifically, we adopt $\text{Pass@}k$ as our primary evaluation metric, which measures whether a model can successfully solve a problem within $k$ attempts. This metric has been widely used to mitigate the unreliability of greedy decoding-based accuracy estimates \citep{hochlehnert2025sober} and to better assess the true capability boundaries of models \citep{chen2021evaluating,chen2025pass,yue2025does,zhu2025surprising}. The unbiased estimator first generates the $n$ responses for per question $\mathbf{x}$ ($n\geq k$),  counts the number of correct responses $c$, then computes the metric as:
\begin{align*}
    \text{Pass@}k=\mathbb{E}_{\mathbf{x}\sim \mathcal{D}}\left[1-\frac{\binom{n-c}{k}}{\binom{n}{k}}\right].
\end{align*}

\subsection{Training Dynamics and Evaluation Results}\label{main:sec td_er}
\begin{figure}[htb]
  \centering
    \begin{subfigure}[b]{0.2445\textwidth}
        \centering
        \includegraphics[height=3.42cm]{   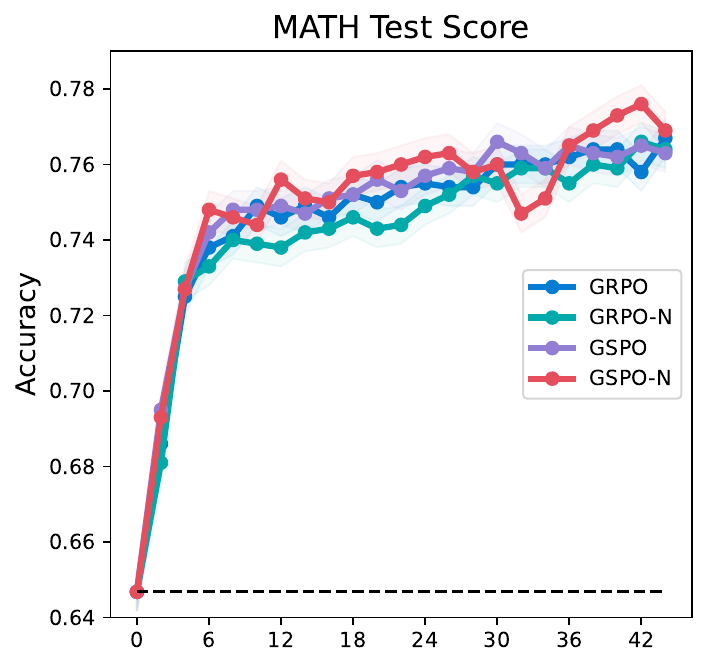}
    \end{subfigure}
    \begin{subfigure}[b]{0.2445\textwidth}
        \centering
        \includegraphics[height=3.42cm]{   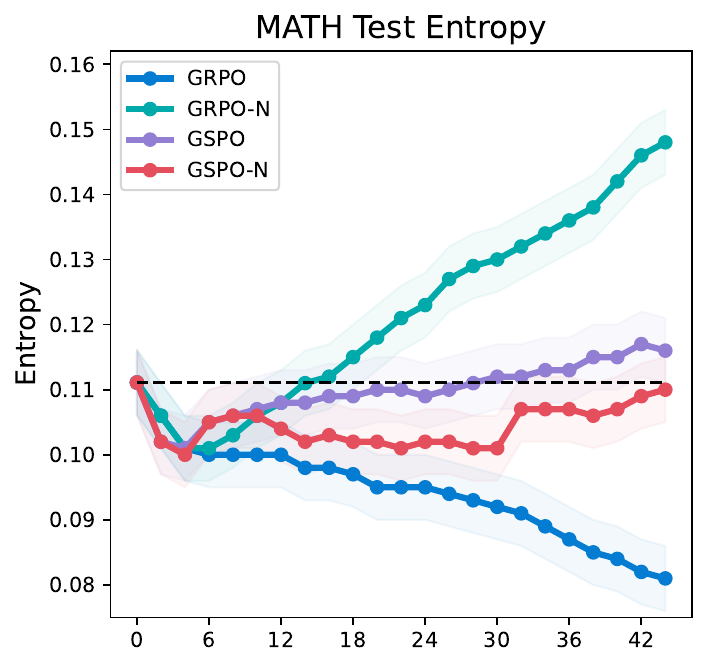}
    \end{subfigure}
    \begin{subfigure}[b]{0.2445\textwidth}
        \centering
        \includegraphics[height=3.42cm]{   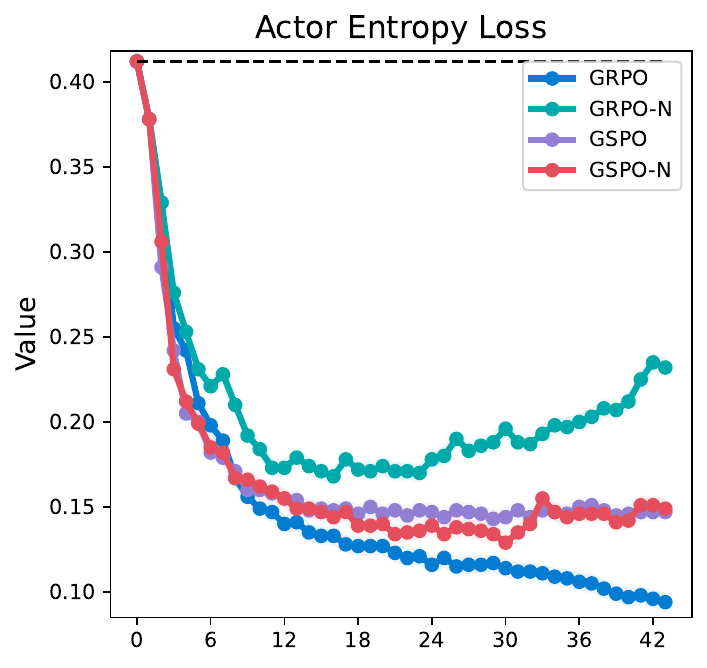}
    \end{subfigure}
    \begin{subfigure}[b]{0.2445\textwidth}
        \centering
        \includegraphics[height=3.42cm]{   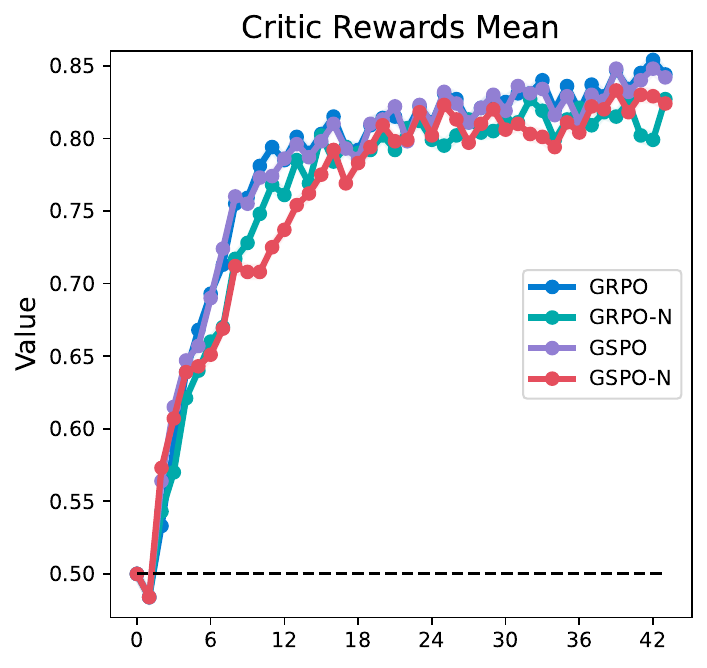}
    \end{subfigure} 
  \caption{Comparison of the training dynamics of GRPO, GRPO-N, GSPO, and GSPO-N on the MATH benchmark across training steps, using the \texttt{Qwen2.5-Math-7B} model with a prompt batch size of 1,024. \textbf{Left Part:} (Left) the greedy decoding accuracy on the MATH test set and (Center Left) the model's entropy on the MATH test set. \textbf{Right Part:} (Center Right) the actor entropy loss and (Right) critic rewards mean during training. GRPO causes the entropy of the base model to collapse over the course of training, suggesting a loss of exploratory capability. In contrast, GRPO-N, GSPO, and GSPO-N all exhibit a pattern where entropy initially decreases and then increases. Notably, the entropy of GRPO-N significantly surpasses that of the base model. All algorithms achieve competitive performance in both greedy decoding accuracy and critic rewards mean.}
  \label{main:figure2}
\end{figure}
\textbf{Training dynamics}. We characterize the training dynamics by monitoring the greedy decoding accuracy and entropy on a held-out MATH test set over the course of training (Figure \ref{main:figure2} for \texttt{Qwen2.5-Math-7B}, Figure \ref{app:figure3} for \texttt{Llama-3.2-3B-Instruct}), together with the actor entropy loss and critic rewards mean during training. As illustrated, GRPO, GRPO-N, GSPO, and GSPO consistently achieve competitive performance in both greedy decoding accuracy and critic rewards mean. Notably, GRPO leads to a rapid and substantial decline in entropy on the MATH test set. In contrast, GRPO-N, GSPO, and GSPO-N all show an initial decrease in entropy, followed by a consistent increase. Importantly, the entropy on the held-out test set under GRPO-N significantly exceeds that of the base model. This divergence indicates that the standard GRPO may limit output diversity and exploratory capability (see Table \ref{table:math_aime_pass_k}), both methods that apply sequence-level importance ratio clipping directly (GSPO and GSPO-N) and those that utilize only relative negative gradients (GRPO-N) help mitigate overconfidence in previously sampled responses. Of particular significance, prior study \citep{cui2025entropy} suggests that policy performance comes at the cost of policy entropy, and is therefore bottlenecked by its exhaustion. Therefore, the model optimized by GRPO-N may be a good baseline and maintain the base model's diversity for prolonging training\footnote{For model with weak prior, we provide the training dynamics and evaluation results of \texttt{LLama-3.2-3B-Instruct} in Appendix \ref{app:more evaluation}. The performance ceiling is related to the base model, yet the key finding remains consistent across different models.}.

\textbf{Performance on ID\&OOD tasks}. As shown in Table \ref{table:math_aime_pass_k}, for model with strong prior (e.g., Qwen models), both GRPO-N and GSPO-N consistently achieve a favorable trade-off across various values of $k$ on both ID tasks (e.g., AMC 2023, AIME 2024, and AIME 2025) and OOD tasks (such as ARC-c and MMLU-Pro). In particular, (1) GSPO-N matches the best Pass@$1$ performance on AMC 2023, AIME 2025, ARC-c and MMLU-Pro. (2) GRPO-N and GSPO-N reliably improve the reasoning performance of the base model on ID tasks for every value of $k$. (3) For
OOD tasks, GRPO-N (GSPO-N) achieves higher Pass@$k$ scores than GRPO (GSPO) across all $k$ values, demonstrating stable performance improvements while largely preserving the diversity of the base model. 

\begin{table*}[ht]
\centering
\setlength{\tabcolsep}{8pt}
\small
\caption{Evaluation results of \texttt{Qwen2.5-Math-7B} on in-domain tasks (AMC 2023, AIME 2024, and AIME 2025) and out-of-domain tasks (ARC-c and  MMLU-Pro). For each $k$, \textbf{bold} and \underline{underlined} numbers indicate the best and second-best results, respectively.}
\label{table:math_aime_pass_k}
\begin{tabular}{lccccccccc}
\toprule
\textbf{Algorithm} & \multicolumn{9}{c}{\textbf{Pass@$k$}} \\
$k$ & 1 & 2 & 4 & 8 & 16 & 32 & 64 & 128 & 256 \\
\midrule
\textbf{} & \multicolumn{9}{c}{\textbf{AMC 2023}} \\
Base Model & 40.4 & 55.6 & 69.1 & 79.4 & 85.9 & 89.5 & 92.1 & 94.6 & 97.5 \\
GRPO & 60.4 & 69.9 & 77.4 & 82.9 & 86.7 & 89.4 & 91.7 & 94.7 & 97.5 \\
\rowcolor{gray!10} GRPO-N & 59.2 & 68.7 & 76.3 & 82.7 & 87.6 & \textbf{92.3} & \textbf{96.3} & \textbf{99.1} & \textbf{100.0} \\
GSPO & \underline{61.1} & \underline{70.5} & \underline{78.0} & \underline{83.9} & \underline{88.1} & 91.6 & 94.4 & 96.2 & 97.5 \\
\rowcolor{gray!10} GSPO-N & \textbf{61.5} & \textbf{71.2} & \textbf{78.5} & \textbf{84.1} & \textbf{88.4} & \underline{91.8} & \underline{94.8} & \underline{97.4} & \underline{100.0} \\
\midrule
\textbf{} & \multicolumn{9}{c}{\textbf{AIME 2024}} \\
Base Model & 13.6 & 21.8 & 30.5 & 37.5 & 43.5 & 49.7 & 55.8 & 61.4 & 66.7 \\
GRPO & 22.6 & 31.5 & 39.5 & 46.2 & 51.9 & 57.3 & 62.9 & \underline{68.9} & 73.3 \\
\rowcolor{gray!10} GRPO-N & \underline{23.6} & \underline{33.4} & 41.8 & 47.5 & 51.9 & 56.7 & 61.8 & 67.3 & \underline{73.3} \\
GSPO & \textbf{25.3} & \textbf{34.7} & \textbf{42.4} & \underline{48.3} & \underline{53.6} & \underline{58.7} & \underline{63.6} & 68.1 & 73.3 \\
\rowcolor{gray!10} GSPO-N & 23.3 & 31.1 & \underline{42.1} & \textbf{48.8} & \textbf{54.3} & \textbf{59.4} & \textbf{64.4} & \textbf{69.6} & \textbf{73.3} \\
\midrule
\textbf{} & \multicolumn{9}{c}{\textbf{AIME 2025}} \\
Base Model & 6.4 & 10.2 & 14.5 & 18.9 & 23.6 & 28.1 & 32.5 & 38.3 & 46.7 \\
GRPO & 9.2 & 13.4 & 17.9 & 22.4 & 26.4 & 29.9 & 33.9 & 39.1 & 46.7 \\
\rowcolor{gray!10} GRPO-N & 9.5 & 14.2 & 19.1 & 23.8 & 28.7 & \underline{34.2} & \textbf{41.6} & \textbf{52.5} & \textbf{66.7} \\
GSPO & \underline{9.6} & \underline{14.3} & \underline{19.2} & \underline{23.9} & \underline{28.8} & 34.2 & 40.9 & \underline{49.5} & \underline{60.0} \\
\rowcolor{gray!10} GSPO-N & \textbf{10.2} & \textbf{14.7} & \textbf{19.6} & \textbf{24.9} & \textbf{29.9} & \textbf{35.0} & \underline{40.9} & 47.3 & 53.3 \\
\midrule
\textbf{} & \multicolumn{9}{c}{\textbf{ARC-c}} \\
Base Model & 35.4 & 54.9 & 73.7 & 86.5 & 93.6 & \underline{96.9} & 98.2 & 99.2 & 100.0 \\
GRPO & \underline{62.3} & 77.4 & \underline{86.6} & \underline{91.6} & \underline{94.2} & 96.3 & \underline{98.2} & \underline{99.5} & \underline{100.0} \\
\rowcolor{gray!10} GRPO-N & 61.7 & \textbf{78.1} & \textbf{88.5} & \textbf{94.3} & \textbf{97.7} & \textbf{99.5} & \textbf{99.9} & \textbf{100.0} & \textbf{100.0} \\
GSPO & 59.9 & 74.1 & 83.2 & 89.1 & 93.2 & 95.7 & 96.7 & 96.9 & 96.9 \\
\rowcolor{gray!10} GSPO-N & \textbf{63.9} & \underline{77.9} & 86.5 & 91.1 & 93.7 & 95.5 & 96.5 & 96.9 & 96.9 \\
\midrule
\textbf{} & \multicolumn{9}{c}{\textbf{MMLU-Pro}} \\
Base Model & 28.1 & 41.4 & 55.1 & 67.6 & \underline{78.1} & \textbf{85.9} & \textbf{91.6} & \textbf{96.2} & \textbf{100.0} \\
GRPO & \underline{40.1} & \underline{52.0} & \textbf{62.4} & \textbf{70.8} & 76.9 & 80.6 & 83.7 & 87.3 & 90.6 \\
    \rowcolor{gray!10} GRPO-N & 38.5 & 49.9 & 60.7 & \underline{70.1} & \textbf{78.1} & \underline{84.3} & \underline{89.0} & \underline{93.7} & \underline{100.0} \\
GSPO & 40.0 & 50.6 & 60.6 & 69.2 & 75.5 & 79.5 & 82.4 & 85.9 & 90.6 \\
\rowcolor{gray!10} GSPO-N & \textbf{41.6} & \textbf{52.3} & \underline{61.8} & 69.7 & 75.9 & 79.9 & 83.6 & 88.1 & 93.8 \\
\bottomrule
\end{tabular}
\end{table*}

In this subsection, we demonstrate that while all algorithms achieve competitive performance, they exhibit distinct behaviors at the entropy level. This observation motivates further analysis of the \textbf{differences among various correct responses} in the next part.
\subsection{Analysis of Different Correct Responses}
\begin{figure}[ht]
        \centering
        \includegraphics[width=\textwidth]{ 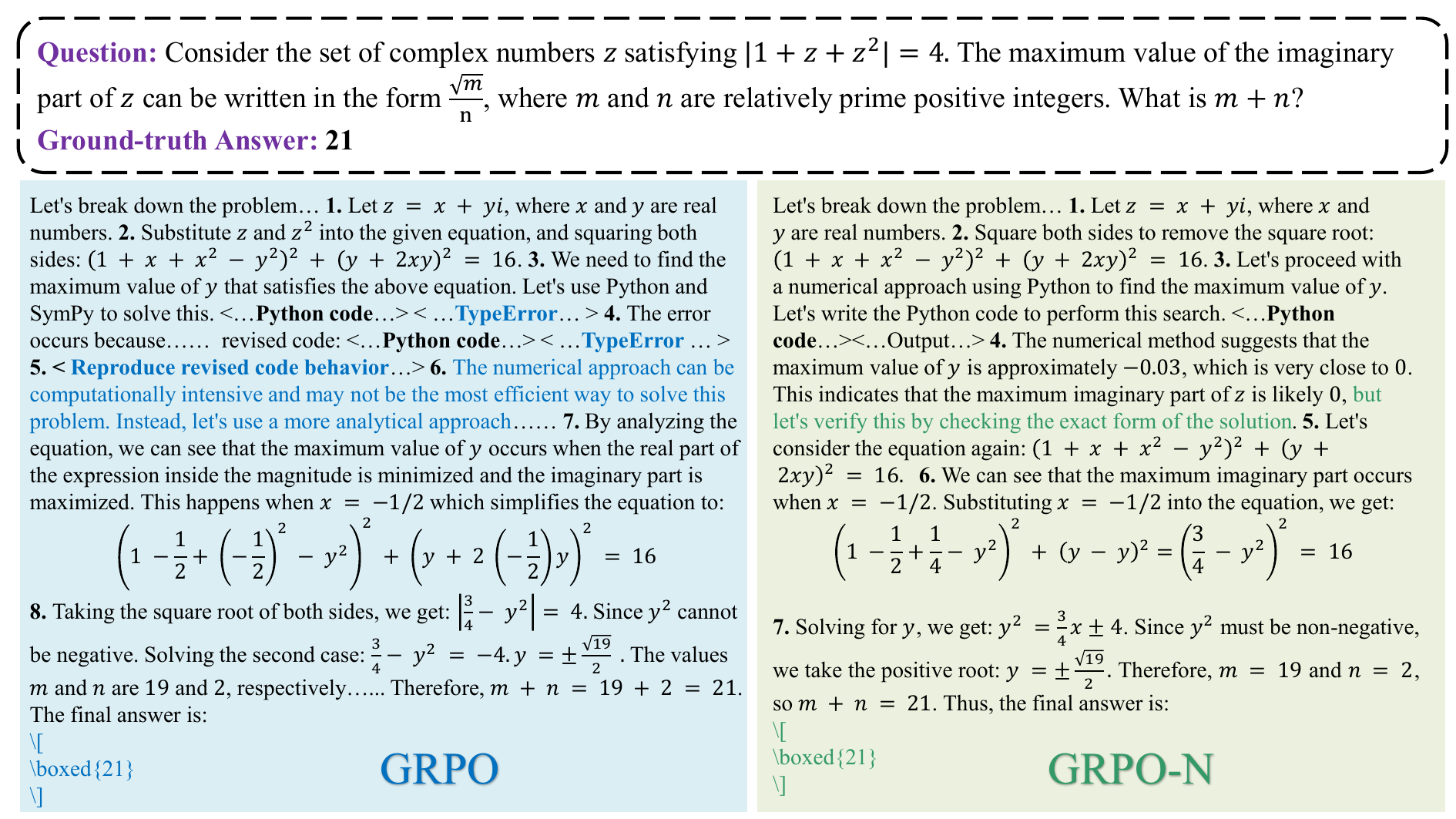}
        \caption{A comparison of the correct responses of GRPO and GRPO-N (a test case from AMC 2023). The key reasoning steps are presented here, see Appendix \ref{app:case} for full procedure.}
        \label{main:figure_case}
\end{figure}

\textbf{Case study}. We present the reasoning processes of GRPO and GRPO-N, as illustrated in Figure \ref{main:figure_case}. Although various responses can yield correct answers through reasoning, their underlying reasoning behaviors differ markedly. Specifically, GRPO tends to repeat similar errors, such as persistently generating code with TypeErrors, and fails to rectify them. In contrast, GRPO-N produces fewer erroneous codes than GRPO. We further examined the responses generated during the training process and found instances where incorrect code was initially produced but later reflected upon and corrected to form the final correct response. GRPO assigns higher probability to the entire trajectory (i.e., the behavior of generating incorrect code is reinforced—a tendency that may significantly affect the model's self-correction ability), whereas GRPO-N mitigates this issue. This necessitates an inquiry into achieving finer-grained control of probability assignments.

\section{Conclusion and Discussion}
\textbf{Conclusion}. Based on a two-stage dynamic view of probability mass allocation, this study resolves the ongoing debate on whether RLVR shrinks or expands LLM reasoning capabilities. We show that initial training favors exploitation, potentially narrowing capability boundaries, while prolonged training encourages exploration, enabling genuine expansion. Theoretically and empirically, we demonstrate that both phenomena occur at different phases. Guided by these findings, one can develop new algorithms to foster more advanced reasoning capabilities. 

\textbf{Discussion}. However, further studies are required on (i) how to design efficient algorithms for fine-grained probability mass allocation; (ii) what kind of base models are more conducive to capability boundary expansion during the RL stage; and (iii) where the ceiling of boundary exploration lies.  We leave these questions for our future work.


\bibliographystyle{plainnat}
\bibliography{reference}

\begin{thebibliography}{88}
\providecommand{\natexlab}[1]{#1}
\providecommand{\url}[1]{\texttt{#1}}
\expandafter\ifx\csname urlstyle\endcsname\relax
  \providecommand{\doi}[1]{doi: #1}\else
  \providecommand{\doi}{doi: \begingroup \urlstyle{rm}\Url}\fi

\bibitem[Abdolmaleki et~al.(2025)Abdolmaleki, Piot, Shahriari, Springenberg, Hertweck, Bloesch, Joshi, Lampe, Oh, Heess, Buchli, and Riedmiller]{abdolmaleki2025learning}
Abbas Abdolmaleki, Bilal Piot, Bobak Shahriari, Jost~Tobias Springenberg, Tim Hertweck, Michael Bloesch, Rishabh Joshi, Thomas Lampe, Junhyuk Oh, Nicolas Heess, Jonas Buchli, and Martin Riedmiller.
\newblock Learning from negative feedback, or positive feedback or both.
\newblock In \emph{International Conference on Learning Representations}, 2025.

\bibitem[An et~al.(2025)An, Xie, Li, Li, Zhang, Gong, Zhong, Xu, Qiu, Wang, and Kong]{Polaris2025}
Chenxin An, Zhihui Xie, Xiaonan Li, Lei Li, Jun Zhang, Shansan Gong, Ming Zhong, Jingjing Xu, Xipeng Qiu, Mingxuan Wang, and Lingpeng Kong.
\newblock Polaris: A post-training recipe for scaling reinforcement learning on advanced reasoning models, 2025.
\newblock URL \url{https://hkunlp.github.io/blog/2025/Polaris}.

\bibitem[Arora et~al.(2019)Arora, Du, Hu, Li, Salakhutdinov, and Wang]{arora2019exact}
Sanjeev Arora, Simon~S Du, Wei Hu, Zhiyuan Li, Russ~R Salakhutdinov, and Ruosong Wang.
\newblock On exact computation with an infinitely wide neural net.
\newblock \emph{Advances in neural information processing systems}, 32, 2019.

\bibitem[Bai et~al.(2025)Bai, Zhang, Qiu, Zhang, Xu, and Li]{bai2025online}
Chenjia Bai, Yang Zhang, Shuang Qiu, Qiaosheng Zhang, Kang Xu, and Xuelong Li.
\newblock Online preference alignment for language models via count-based exploration.
\newblock In \emph{The Thirteenth International Conference on Learning Representations}, 2025.

\bibitem[Bai et~al.(2022)Bai, Jones, Ndousse, Askell, Chen, DasSarma, Drain, Fort, Ganguli, Henighan, et~al.]{bai2022training}
Yuntao Bai, Andy Jones, Kamal Ndousse, Amanda Askell, Anna Chen, Nova DasSarma, Dawn Drain, Stanislav Fort, Deep Ganguli, Tom Henighan, et~al.
\newblock Training a helpful and harmless assistant with reinforcement learning from human feedback.
\newblock \emph{arXiv preprint arXiv:2204.05862}, 2022.

\bibitem[Cai et~al.(2025)Cai, Wang, Chen, and Dhingra]{cai2025backtrack}
Hongyi~James Cai, Junlin Wang, Xiaoyin Chen, and Bhuwan Dhingra.
\newblock How much backtracking is enough? exploring the interplay of sft and rl in enhancing llm reasoning.
\newblock \emph{arXiv preprint arXiv:2505.24273}, 2025.

\bibitem[Chen et~al.(2021)Chen, Tworek, Jun, Yuan, Pinto, Kaplan, Edwards, Burda, Joseph, Brockman, et~al.]{chen2021evaluating}
Mark Chen, Jerry Tworek, Heewoo Jun, Qiming Yuan, Henrique Ponde De~Oliveira Pinto, Jared Kaplan, Harri Edwards, Yuri Burda, Nicholas Joseph, Greg Brockman, et~al.
\newblock Evaluating large language models trained on code.
\newblock \emph{arXiv preprint arXiv:2107.03374}, 2021.

\bibitem[Chen et~al.(2025{\natexlab{a}})Chen, Li, Li, Chen, and Lin]{chen2025spectral}
Peter Chen, Xiaopeng Li, Ziniu Li, Xi~Chen, and Tianyi Lin.
\newblock Spectral policy optimization: Coloring your incorrect reasoning in grpo.
\newblock \emph{arXiv preprint arXiv:2505.11595}, 2025{\natexlab{a}}.

\bibitem[Chen et~al.(2025{\natexlab{b}})Chen, Qin, Wu, Ling, Ye, Zhao, and Shi]{chen2025pass}
Zhipeng Chen, Xiaobo Qin, Youbin Wu, Yue Ling, Qinghao Ye, Wayne~Xin Zhao, and Guang Shi.
\newblock Pass@ k training for adaptively balancing exploration and exploitation of large reasoning models.
\newblock \emph{arXiv preprint arXiv:2508.10751}, 2025{\natexlab{b}}.

\bibitem[Chu et~al.(2025)Chu, Huang, Zhang, Wei, and Wang]{gpg}
Xiangxiang Chu, Hailang Huang, Xiao Zhang, Fei Wei, and Yong Wang.
\newblock Gpg: A simple and strong reinforcement learning baseline for model reasoning.
\newblock \emph{arXiv preprint arXiv:2504.02546}, 2025.

\bibitem[Clark et~al.(2018)Clark, Cowhey, Etzioni, Khot, Sabharwal, Schoenick, and Tafjord]{arc}
Peter Clark, Isaac Cowhey, Oren Etzioni, Tushar Khot, Ashish Sabharwal, Carissa Schoenick, and Oyvind Tafjord.
\newblock Think you have solved question answering? try arc, the ai2 reasoning challenge.
\newblock \emph{arXiv preprint arXiv:1803.05457}, 2018.

\bibitem[Cui et~al.(2025)Cui, Zhang, Chen, Yuan, Wang, Zuo, Li, Fan, Chen, Chen, et~al.]{cui2025entropy}
Ganqu Cui, Yuchen Zhang, Jiacheng Chen, Lifan Yuan, Zhi Wang, Yuxin Zuo, Haozhan Li, Yuchen Fan, Huayu Chen, Weize Chen, et~al.
\newblock The entropy mechanism of reinforcement learning for reasoning language models.
\newblock \emph{arXiv preprint arXiv:2505.22617}, 2025.

\bibitem[Dang et~al.(2025{\natexlab{a}})Dang, Baek, Kolter, and Raghunathan]{dang2025assessing}
Xingyu Dang, Christina Baek, J~Zico Kolter, and Aditi Raghunathan.
\newblock Assessing diversity collapse in reasoning.
\newblock In \emph{Scaling Self-Improving Foundation Models without Human Supervision}, 2025{\natexlab{a}}.

\bibitem[Dang et~al.(2025{\natexlab{b}})Dang, Baek, Wen, Kolter, and Raghunathan]{dang2025weight}
Xingyu Dang, Christina Baek, Kaiyue Wen, J~Zico Kolter, and Aditi Raghunathan.
\newblock Weight ensembling improves reasoning in language models.
\newblock In \emph{Second Conference on Language Modeling}, 2025{\natexlab{b}}.

\bibitem[Deng et~al.(2025)Deng, Ren, Li, Sutherland, Li, and Thrampoulidis]{deng2025effectnegativegradientgroup}
Wenlong Deng, Yi~Ren, Muchen Li, Danica~J Sutherland, Xiaoxiao Li, and Christos Thrampoulidis.
\newblock On the effect of negative gradient in group relative deep reinforcement optimization.
\newblock \emph{arXiv preprint arXiv:2505.18830}, 2025.

\bibitem[Deng et~al.(2021)Deng, He, and Su]{deng2021toward}
Zhun Deng, Hangfeng He, and Weijie Su.
\newblock Toward better generalization bounds with locally elastic stability.
\newblock In \emph{International Conference on Machine Learning}, pages 2590--2600, 2021.

\bibitem[Dong et~al.(2023)Dong, Xiong, Goyal, Zhang, Chow, Pan, Diao, Zhang, SHUM, and Zhang]{dong2023raft}
Hanze Dong, Wei Xiong, Deepanshu Goyal, Yihan Zhang, Winnie Chow, Rui Pan, Shizhe Diao, Jipeng Zhang, KaShun SHUM, and Tong Zhang.
\newblock {RAFT}: Reward ranked finetuning for generative foundation model alignment.
\newblock \emph{Transactions on Machine Learning Research}, 2023.
\newblock ISSN 2835-8856.

\bibitem[Dong et~al.(2025)Dong, Jiang, Tao, Liu, Zhang, Mou, Cao, Ma, Chen, Li, Jin, Huang, Li, and Li]{rlplus}
Yihong Dong, Xue Jiang, Yongding Tao, Huanyu Liu, Kechi Zhang, Lili Mou, Rongyu Cao, Yingwei Ma, Jue Chen, Binhua Li, Zhi Jin, Fei Huang, Yongbin Li, and Ge~Li.
\newblock Rl-plus: Countering capability boundary collapse of llms in reinforcement learning with hybrid-policy optimization.
\newblock \emph{arXiv preprint arXiv:2508.00222}, 2025.

\bibitem[Fort et~al.(2019)Fort, Nowak, Jastrzebski, and Narayanan]{fort2019stiffness}
Stanislav Fort, Pawe{\l}~Krzysztof Nowak, Stanislaw Jastrzebski, and Srini Narayanan.
\newblock Stiffness: A new perspective on generalization in neural networks.
\newblock \emph{arXiv preprint arXiv:1901.09491}, 2019.

\bibitem[Gandhi et~al.(2025)Gandhi, Chakravarthy, Singh, Lile, and Goodman]{gandhi2025cognitive}
Kanishk Gandhi, Ayush~K Chakravarthy, Anikait Singh, Nathan Lile, and Noah Goodman.
\newblock Cognitive behaviors that enable self-improving reasoners, or, four habits of highly effective {ST}ars.
\newblock In \emph{Second Conference on Language Modeling}, 2025.

\bibitem[Guo et~al.(2025)Guo, Yang, Zhang, Song, Zhang, Xu, Zhu, Ma, Wang, Bi, et~al.]{guo2025deepseek}
Daya Guo, Dejian Yang, Haowei Zhang, Junxiao Song, Ruoyu Zhang, Runxin Xu, Qihao Zhu, Shirong Ma, Peiyi Wang, Xiao Bi, et~al.
\newblock Deepseek-r1: Incentivizing reasoning capability in llms via reinforcement learning.
\newblock \emph{arXiv preprint arXiv:2501.12948}, 2025.

\bibitem[Guo et~al.(2024)Guo, Ren, Albrecht, and Smith]{guo2024lpntk}
Shangmin Guo, Yi~Ren, Stefano~V Albrecht, and Kenny Smith.
\newblock lp{NTK}: Better generalisation with less data via sample interaction during learning.
\newblock In \emph{The Twelfth International Conference on Learning Representations}, 2024.

\bibitem[Hao et~al.(2024)Hao, Sukhbaatar, Su, Li, Hu, Weston, and Tian]{hao2024training}
Shibo Hao, Sainbayar Sukhbaatar, DiJia Su, Xian Li, Zhiting Hu, Jason Weston, and Yuandong Tian.
\newblock Training large language models to reason in a continuous latent space.
\newblock \emph{arXiv preprint arXiv:2412.06769}, 2024.

\bibitem[He et~al.(2025{\natexlab{a}})He, Fried, and Welleck]{he2025rewarding}
Andre He, Daniel Fried, and Sean Welleck.
\newblock Rewarding the unlikely: Lifting grpo beyond distribution sharpening.
\newblock \emph{arXiv preprint arXiv:2506.02355}, 2025{\natexlab{a}}.

\bibitem[He and Su(2020)]{he2020local}
Hangfeng He and Weijie Su.
\newblock The local elasticity of neural networks.
\newblock In \emph{International Conference on Learning Representations}, 2020.

\bibitem[He et~al.(2025{\natexlab{b}})He, Luo, Zhang, Yang, and Qiu]{he2025deltalnormalizationrethink}
Zhiyuan He, Xufang Luo, Yike Zhang, Yuqing Yang, and Lili Qiu.
\newblock $\delta l$ normalization: Rethink loss aggregation in rlvr.
\newblock \emph{arXiv preprint arXiv:2509.07558}, 2025{\natexlab{b}}.

\bibitem[Hendrycks et~al.(2021)Hendrycks, Burns, Kadavath, Arora, Basart, Tang, Song, and Steinhardt]{hendrycks2021measuring}
Dan Hendrycks, Collin Burns, Saurav Kadavath, Akul Arora, Steven Basart, Eric Tang, Dawn Song, and Jacob Steinhardt.
\newblock Measuring mathematical problem solving with the {MATH} dataset.
\newblock In \emph{Thirty-fifth Conference on Neural Information Processing Systems Datasets and Benchmarks Track (Round 2)}, 2021.

\bibitem[Hochlehnert et~al.(2025)Hochlehnert, Bhatnagar, Udandarao, Albanie, Prabhu, and Bethge]{hochlehnert2025sober}
Andreas Hochlehnert, Hardik Bhatnagar, Vishaal Udandarao, Samuel Albanie, Ameya Prabhu, and Matthias Bethge.
\newblock A sober look at progress in language model reasoning: Pitfalls and paths to reproducibility.
\newblock \emph{arXiv preprint arXiv:2504.07086}, 2025.

\bibitem[Hu et~al.(2025{\natexlab{a}})Hu, Liu, Xu, and Shen]{hu2025reinforceefficientrlhfalgorithm}
Jian Hu, Jason~Klein Liu, Haotian Xu, and Wei Shen.
\newblock Reinforce++: An efficient rlhf algorithm with robustness to both prompt and reward models.
\newblock \emph{arXiv preprint arXiv:2501.03262}, 2025{\natexlab{a}}.

\bibitem[Hu et~al.(2025{\natexlab{b}})Hu, Zhang, Han, Jiang, Zhang, and Shum]{hu2025openreasonerzeroopensourceapproach}
Jingcheng Hu, Yinmin Zhang, Qi~Han, Daxin Jiang, Xiangyu Zhang, and Heung-Yeung Shum.
\newblock Open-reasoner-zero: An open source approach to scaling up reinforcement learning on the base model.
\newblock \emph{arXiv preprint arXiv:2503.24290}, 2025{\natexlab{b}}.

\bibitem[Jacot et~al.(2018)Jacot, Gabriel, and Hongler]{jacot2018neural}
Arthur Jacot, Franck Gabriel, and Cl{\'e}ment Hongler.
\newblock Neural tangent kernel: Convergence and generalization in neural networks.
\newblock \emph{Advances in neural information processing systems}, 31, 2018.

\bibitem[Jaech et~al.(2024)Jaech, Kalai, Lerer, Richardson, El-Kishky, Low, Helyar, Madry, Beutel, Carney, et~al.]{jaech2024openai}
Aaron Jaech, Adam Kalai, Adam Lerer, Adam Richardson, Ahmed El-Kishky, Aiden Low, Alec Helyar, Aleksander Madry, Alex Beutel, Alex Carney, et~al.
\newblock Openai o1 system card.
\newblock \emph{arXiv preprint arXiv:2412.16720}, 2024.

\bibitem[Li et~al.(2025)Li, Zhang, Xu, Xue, Ao, and He]{gapo}
Chengao Li, Hanyu Zhang, Yunkun Xu, Hongyan Xue, Xiang Ao, and Qing He.
\newblock Gradient-adaptive policy optimization: Towards multi-objective alignment of large language models.
\newblock \emph{arXiv preprint arXiv:2507.01915}, 2025.

\bibitem[Liu(2025)]{liu2025intro}
Jiacai Liu.
\newblock Brief introduction of policy gradient in llm reasoning.
\newblock https://notion.so/Brief-Introduction-of-Policy-Gradient-In-LLM-Reasoning-1c04795a3e8b805abbd6ccc9f1a34ac0LiuLiu, 2025.

\bibitem[Liu and Zhang(2025)]{liu2025code}
Jiawei Liu and Lingming Zhang.
\newblock Code-r1: Reproducing r1 for code with reliable rewards.
\newblock \emph{arXiv preprint arXiv:2503.18470}, 3, 2025.

\bibitem[Liu et~al.(2025{\natexlab{a}})Liu, Diao, Hu, Lu, Dong, Zhang, Bukharin, Zhang, Zeng, Sreedhar, et~al.]{liu2025scaling}
Mingjie Liu, Shizhe Diao, Jian Hu, Ximing Lu, Xin Dong, Hao Zhang, Alexander Bukharin, Shaokun Zhang, Jiaqi Zeng, Makesh~Narsimhan Sreedhar, et~al.
\newblock Scaling up rl: Unlocking diverse reasoning in llms via prolonged training.
\newblock \emph{arXiv preprint arXiv:2507.12507}, 2025{\natexlab{a}}.

\bibitem[Liu et~al.(2025{\natexlab{b}})Liu, Diao, Lu, Hu, Dong, Choi, Kautz, and Dong]{liu2025prorl}
Mingjie Liu, Shizhe Diao, Ximing Lu, Jian Hu, Xin Dong, Yejin Choi, Jan Kautz, and Yi~Dong.
\newblock Prorl: Prolonged reinforcement learning expands reasoning boundaries in large language models.
\newblock \emph{arXiv preprint arXiv:2505.24864}, 2025{\natexlab{b}}.

\bibitem[Liu et~al.(2025{\natexlab{c}})Liu, Chen, Li, Qi, Pang, Du, Lee, and Lin]{liu2025understandingr1zeroliketrainingcritical}
Zichen Liu, Changyu Chen, Wenjun Li, Penghui Qi, Tianyu Pang, Chao Du, Wee~Sun Lee, and Min Lin.
\newblock Understanding r1-zero-like training: A critical perspective.
\newblock \emph{arXiv preprint arXiv:2503.20783}, 2025{\natexlab{c}}.

\bibitem[Luo et~al.(2025{\natexlab{a}})Luo, Tan, Huang, Patel, Ariyak, Wu, Shi, Xin, Cai, Weber, et~al.]{luo2025deepcoder}
Michael Luo, Sijun Tan, Roy Huang, Ameen Patel, Alpay Ariyak, Qingyang Wu, Xiaoxiang Shi, Rachel Xin, Colin Cai, Maurice Weber, et~al.
\newblock Deepcoder: A fully open-source 14b coder at o3-mini level.
\newblock \emph{Notion Blog}, 2025{\natexlab{a}}.

\bibitem[Luo et~al.(2025{\natexlab{b}})Luo, Tan, Wong, Shi, Tang, Roongta, Cai, Luo, Zhang, Li, et~al.]{luo2025deepscaler}
Michael Luo, Sijun Tan, Justin Wong, Xiaoxiang Shi, William~Y Tang, Manan Roongta, Colin Cai, Jeffrey Luo, Tianjun Zhang, Li~Erran Li, et~al.
\newblock Deepscaler: Surpassing o1-preview with a 1.5 b model by scaling rl.
\newblock \emph{Notion Blog}, 2025{\natexlab{b}}.

\bibitem[Ma et~al.(2025{\natexlab{a}})Ma, Liang, Qiang, Tang, Ma, Wong, Niu, Shen, He, Cui, and Zhang]{ma2025cant}
Lu~Ma, Hao Liang, Meiyi Qiang, Lexiang Tang, Xiaochen Ma, Zhen~Hao Wong, Junbo Niu, Chengyu Shen, Runming He, Bin Cui, and Wentao Zhang.
\newblock Learning what reinforcement learning can't: Interleaved online fine-tuning for hardest questions.
\newblock \emph{arXiv preprint arXiv:2506.07527}, 2025{\natexlab{a}}.

\bibitem[Ma et~al.(2025{\natexlab{b}})Ma, Liang, Qiang, Tang, Ma, Wong, Niu, Shen, He, Cui, et~al.]{ma2025learning}
Lu~Ma, Hao Liang, Meiyi Qiang, Lexiang Tang, Xiaochen Ma, Zhen~Hao Wong, Junbo Niu, Chengyu Shen, Runming He, Bin Cui, et~al.
\newblock Learning what reinforcement learning can't: Interleaved online fine-tuning for hardest questions.
\newblock \emph{arXiv preprint arXiv:2506.07527}, 2025{\natexlab{b}}.

\bibitem[O'Mahony et~al.(2024)O'Mahony, Grinsztajn, Schoelkopf, and Biderman]{o'mahony2024attributing}
Laura O'Mahony, Leo Grinsztajn, Hailey Schoelkopf, and Stella Biderman.
\newblock Attributing mode collapse in the fine-tuning of large language models.
\newblock In \emph{ICLR 2024 Workshop on Mathematical and Empirical Understanding of Foundation Models}, 2024.

\bibitem[{OpenAI}(2025)]{openai2025o3}
{OpenAI}.
\newblock Introducing openai o3 and o4-mini, 2025.
\newblock Accessed: April 16, 2025.

\bibitem[(OpenAI)(2025)]{secondhalf}
Shunyu~Yao (OpenAI).
\newblock The second half.
\newblock \url{https://ysymyth.github.io/The-Second-Half/}, 2025.

\bibitem[Ouyang et~al.(2022)Ouyang, Wu, Jiang, Almeida, Wainwright, Mishkin, Zhang, Agarwal, Slama, Ray, et~al.]{ouyang2022traininglanguagemodelsfollow}
Long Ouyang, Jeffrey Wu, Xu~Jiang, Diogo Almeida, Carroll Wainwright, Pamela Mishkin, Chong Zhang, Sandhini Agarwal, Katarina Slama, Alex Ray, et~al.
\newblock Training language models to follow instructions with human feedback.
\newblock \emph{Advances in neural information processing systems}, pages 27730--27744, 2022.

\bibitem[Pruthi et~al.(2020)Pruthi, Liu, Kale, and Sundararajan]{pruthi2020estimating}
Garima Pruthi, Frederick Liu, Satyen Kale, and Mukund Sundararajan.
\newblock Estimating training data influence by tracing gradient descent.
\newblock \emph{Advances in Neural Information Processing Systems}, 33:\penalty0 19920--19930, 2020.

\bibitem[Qian et~al.(2025)Qian, Liu, Wen, Bai, Liu, and Shao]{qian2025demystifying}
Chen Qian, Dongrui Liu, Haochen Wen, Zhen Bai, Yong Liu, and Jing Shao.
\newblock Demystifying reasoning dynamics with mutual information: Thinking tokens are information peaks in llm reasoning.
\newblock \emph{arXiv preprint arXiv:2506.02867}, 2025.

\bibitem[Rafailov et~al.(2024)Rafailov, Sharma, Mitchell, Ermon, Manning, and Finn]{dpo}
Rafael Rafailov, Archit Sharma, Eric Mitchell, Stefano Ermon, Christopher~D. Manning, and Chelsea Finn.
\newblock Direct preference optimization: Your language model is secretly a reward model.
\newblock \emph{arXiv preprint arXiv:2305.18290}, 2024.

\bibitem[Ren and Sutherland(2025)]{renlearning}
Yi~Ren and Danica~J Sutherland.
\newblock Learning dynamics of llm finetuning.
\newblock In \emph{International Conference on Learning Representations}, 2025.

\bibitem[Schulman et~al.(2015)Schulman, Moritz, Levine, Jordan, and Abbeel]{schulman2018highdimensionalcontinuouscontrolusing}
John Schulman, Philipp Moritz, Sergey Levine, Michael Jordan, and Pieter Abbeel.
\newblock High-dimensional continuous control using generalized advantage estimation.
\newblock \emph{arXiv preprint arXiv:1506.02438}, 2015.

\bibitem[Schulman et~al.(2017)Schulman, Wolski, Dhariwal, Radford, and Klimov]{schulman2017proximalpolicyoptimizationalgorithms}
John Schulman, Filip Wolski, Prafulla Dhariwal, Alec Radford, and Oleg Klimov.
\newblock Proximal policy optimization algorithms.
\newblock \emph{arXiv preprint arXiv:1707.06347}, 2017.

\bibitem[Shah et~al.(2025)Shah, Rushton, Singla, Parmar, Smith, Vanjani, Vaswani, Chaluvaraju, Hojel, Ma, et~al.]{shah2025rethinking}
Darsh~J Shah, Peter Rushton, Somanshu Singla, Mohit Parmar, Kurt Smith, Yash Vanjani, Ashish Vaswani, Adarsh Chaluvaraju, Andrew Hojel, Andrew Ma, et~al.
\newblock Rethinking reflection in pre-training.
\newblock \emph{arXiv preprint arXiv:2504.04022}, 2025.

\bibitem[Shao et~al.(2024)Shao, Wang, Zhu, Xu, Song, Bi, Zhang, Zhang, Li, Wu, et~al.]{deepseekmath}
Zhihong Shao, Peiyi Wang, Qihao Zhu, Runxin Xu, Junxiao Song, Xiao Bi, Haowei Zhang, Mingchuan Zhang, YK~Li, Yang Wu, et~al.
\newblock Deepseekmath: Pushing the limits of mathematical reasoning in open language models.
\newblock \emph{arXiv preprint arXiv:2402.03300}, 2024.

\bibitem[Sheng et~al.(2025)Sheng, Zhang, Ye, Wu, Zhang, Zhang, Peng, Lin, and Wu]{sheng2025hybridflow}
Guangming Sheng, Chi Zhang, Zilingfeng Ye, Xibin Wu, Wang Zhang, Ru~Zhang, Yanghua Peng, Haibin Lin, and Chuan Wu.
\newblock Hybridflow: A flexible and efficient rlhf framework.
\newblock In \emph{Proceedings of the Twentieth European Conference on Computer Systems}, pages 1279--1297, 2025.

\bibitem[Song et~al.(2025)Song, Zhang, Eisenach, Kakade, Foster, and Ghai]{song2025mind}
Yuda Song, Hanlin Zhang, Carson Eisenach, Sham~M. Kakade, Dean Foster, and Udaya Ghai.
\newblock Mind the gap: Examining the self-improvement capabilities of large language models.
\newblock In \emph{The Thirteenth International Conference on Learning Representations}, 2025.

\bibitem[Sun et~al.(2025)Sun, Cao, Huang, Bai, Hajishirzi, Dziri, and Song]{sun2025deltacodedoesrlunlock}
Yiyou Sun, Yuhan Cao, Pohao Huang, Haoyue Bai, Hannaneh Hajishirzi, Nouha Dziri, and Dawn Song.
\newblock Delta-code: How does rl unlock and transfer new programming algorithms in llms?
\newblock \emph{arXiv preprint arXiv:2509.21016}, 2025.

\bibitem[Swamy et~al.(2025)Swamy, Choudhury, Sun, Wu, and Bagnell]{swamy2025all}
Gokul Swamy, Sanjiban Choudhury, Wen Sun, Zhiwei~Steven Wu, and J~Andrew Bagnell.
\newblock All roads lead to likelihood: The value of reinforcement learning in fine-tuning.
\newblock \emph{arXiv preprint arXiv:2503.01067}, 2025.

\bibitem[Touvron et~al.(2023)Touvron, Lavril, Izacard, Martinet, Lachaux, Lacroix, Rozi{\`e}re, Goyal, Hambro, Azhar, et~al.]{touvron2023llama}
Hugo Touvron, Thibaut Lavril, Gautier Izacard, Xavier Martinet, Marie-Anne Lachaux, Timoth{\'e}e Lacroix, Baptiste Rozi{\`e}re, Naman Goyal, Eric Hambro, Faisal Azhar, et~al.
\newblock Llama: Open and efficient foundation language models.
\newblock \emph{arXiv preprint arXiv:2302.13971}, 2023.

\bibitem[Trung et~al.(2024)Trung, Zhang, Jie, Sun, Jin, and Li]{trung-etal-2024-reft}
Luong Trung, Xinbo Zhang, Zhanming Jie, Peng Sun, Xiaoran Jin, and Hang Li.
\newblock {R}e{FT}: Reasoning with reinforced fine-tuning.
\newblock In Lun-Wei Ku, Andre Martins, and Vivek Srikumar, editors, \emph{Proceedings of the 62nd Annual Meeting of the Association for Computational Linguistics (Volume 1: Long Papers)}, pages 7601--7614, 2024.

\bibitem[Wang et~al.(2025)Wang, Xu, Liu, Wu, Lin, and Chen]{wang2025emergent}
Haozhe Wang, Qixin Xu, Che Liu, Junhong Wu, Fangzhen Lin, and Wenhu Chen.
\newblock Emergent hierarchical reasoning in llms through reinforcement learning.
\newblock \emph{arXiv preprint arXiv:2509.03646}, 2025.

\bibitem[Wang et~al.(2024)Wang, Ma, Zhang, Ni, Chandra, Guo, Ren, Arulraj, He, Jiang, et~al.]{mmlu}
Yubo Wang, Xueguang Ma, Ge~Zhang, Yuansheng Ni, Abhranil Chandra, Shiguang Guo, Weiming Ren, Aaran Arulraj, Xuan He, Ziyan Jiang, et~al.
\newblock Mmlu-pro: A more robust and challenging multi-task language understanding benchmark.
\newblock \emph{Advances in Neural Information Processing Systems}, 37:\penalty0 95266--95290, 2024.

\bibitem[Wei et~al.(2022)Wei, Wang, Schuurmans, Bosma, Xia, Chi, Le, Zhou, et~al.]{cot}
Jason Wei, Xuezhi Wang, Dale Schuurmans, Maarten Bosma, Fei Xia, Ed~Chi, Quoc~V Le, Denny Zhou, et~al.
\newblock Chain-of-thought prompting elicits reasoning in large language models.
\newblock \emph{Advances in neural information processing systems}, 35:\penalty0 24824--24837, 2022.

\bibitem[Wen et~al.(2025)Wen, Liu, Zheng, Xu, Ye, Wu, Liang, Wang, Li, Miao, et~al.]{wen2025reinforcement}
Xumeng Wen, Zihan Liu, Shun Zheng, Zhijian Xu, Shengyu Ye, Zhirong Wu, Xiao Liang, Yang Wang, Junjie Li, Ziming Miao, et~al.
\newblock Reinforcement learning with verifiable rewards implicitly incentivizes correct reasoning in base llms.
\newblock \emph{arXiv preprint arXiv:2506.14245}, 2025.

\bibitem[Williams(1992)]{REINFORCE}
Ronald~J. Williams.
\newblock Simple statistical gradient-following algorithms for connectionist reinforcement learning.
\newblock \emph{Mach. Learn.}, page 229–256, 1992.
\newblock ISSN 0885-6125.

\bibitem[Wu et~al.(2025{\natexlab{a}})Wu, Xuan, Lu, Harchaoui, and Choi]{wu2025invisible}
Fang Wu, Weihao Xuan, Ximing Lu, Zaid Harchaoui, and Yejin Choi.
\newblock The invisible leash: Why rlvr may not escape its origin.
\newblock \emph{arXiv preprint arXiv:2507.14843}, 2025{\natexlab{a}}.

\bibitem[Wu et~al.(2025{\natexlab{b}})Wu, Liao, Feng, Zhang, Wen, Shao, Xu, and Tao]{TAPO}
Jinyang Wu, Chonghua Liao, Mingkuan Feng, Shuai Zhang, Zhengqi Wen, Pengpeng Shao, Huazhe Xu, and Jianhua Tao.
\newblock Thought-augmented policy optimization: Bridging external guidance and internal capabilities.
\newblock \emph{arXiv preprint arXiv:2505.15692}, 2025{\natexlab{b}}.

\bibitem[Xia et~al.(2024)Xia, Malladi, Gururangan, Arora, and Chen]{xia2024less}
Mengzhou Xia, Sadhika Malladi, Suchin Gururangan, Sanjeev Arora, and Danqi Chen.
\newblock Less: Selecting influential data for targeted instruction tuning.
\newblock \emph{arXiv preprint arXiv:2402.04333}, 2024.

\bibitem[Xiong et~al.(2025{\natexlab{a}})Xiong, Zhou, Ye, and Dou]{xiong2025aapo}
Jian Xiong, Jingbo Zhou, Jingyong Ye, and Dejing Dou.
\newblock Aapo: Enhance the reasoning capabilities of llms with advantage momentum.
\newblock \emph{arXiv preprint arXiv:2505.14264}, 2025{\natexlab{a}}.

\bibitem[Xiong et~al.(2025{\natexlab{b}})Xiong, Yao, Xu, Pang, Wang, Sahoo, Li, Jiang, Zhang, Xiong, et~al.]{xiong2025minimalist}
Wei Xiong, Jiarui Yao, Yuhui Xu, Bo~Pang, Lei Wang, Doyen Sahoo, Junnan Li, Nan Jiang, Tong Zhang, Caiming Xiong, et~al.
\newblock A minimalist approach to llm reasoning: from rejection sampling to reinforce.
\newblock \emph{arXiv preprint arXiv:2504.11343}, 2025{\natexlab{b}}.

\bibitem[Yan et~al.(2025)Yan, Li, Hu, Wang, Cui, Qu, Cheng, and Zhang]{LUFFY}
Jianhao Yan, Yafu Li, Zican Hu, Zhi Wang, Ganqu Cui, Xiaoye Qu, Yu~Cheng, and Yue Zhang.
\newblock Learning to reason under off-policy guidance.
\newblock \emph{arXiv preprint arXiv:2504.14945}, 2025.

\bibitem[Yang et~al.(2024)Yang, Zhang, Hui, Gao, Yu, Li, Liu, Tu, Zhou, Lin, et~al.]{yang2024qwen2}
An~Yang, Beichen Zhang, Binyuan Hui, Bofei Gao, Bowen Yu, Chengpeng Li, Dayiheng Liu, Jianhong Tu, Jingren Zhou, Junyang Lin, et~al.
\newblock Qwen2. 5-math technical report: Toward mathematical expert model via self-improvement.
\newblock \emph{arXiv preprint arXiv:2409.12122}, 2024.

\bibitem[Yang et~al.(2025{\natexlab{a}})Yang, Ye, Jiang, Hu, Li, Deng, and Jiang]{yang2025unearthinggemsstonespolicy}
Zhaohui Yang, Yuxiao Ye, Shilei Jiang, Chen Hu, Linjing Li, Shihong Deng, and Daxin Jiang.
\newblock Unearthing gems from stones: Policy optimization with negative sample augmentation for llm reasoning.
\newblock \emph{arXiv preprint arXiv:2505.14403}, 2025{\natexlab{a}}.

\bibitem[Yang et~al.(2025{\natexlab{b}})Yang, Luo, Wang, Han, He, Li, and Xu]{yang2025letlowoverdominate}
Zhihe Yang, Xufang Luo, Zilong Wang, Dongqi Han, Zhiyuan He, Dongsheng Li, and Yunjian Xu.
\newblock Do not let low-probability tokens over-dominate in rl for llms.
\newblock \emph{arXiv preprint arXiv:2505.12929}, 2025{\natexlab{b}}.

\bibitem[Yu et~al.(2025{\natexlab{a}})Yu, Cheng, Wu, and Xing]{gpo}
Jiahao Yu, Zelei Cheng, Xian Wu, and Xinyu Xing.
\newblock Gpo: Learning from critical steps to improve llm reasoning.
\newblock \emph{arXiv preprint arXiv:2509.16456}, 2025{\natexlab{a}}.

\bibitem[Yu et~al.(2025{\natexlab{b}})Yu, Zhang, Zhu, Yuan, Zuo, Yue, Dai, Fan, Liu, Liu, et~al.]{yu2025dapoopensourcellmreinforcement}
Qiying Yu, Zheng Zhang, Ruofei Zhu, Yufeng Yuan, Xiaochen Zuo, Yu~Yue, Weinan Dai, Tiantian Fan, Gaohong Liu, Lingjun Liu, et~al.
\newblock Dapo: An open-source llm reinforcement learning system at scale.
\newblock \emph{arXiv preprint arXiv:2503.14476}, 2025{\natexlab{b}}.

\bibitem[Yuan et~al.(2025)Yuan, Chen, Zhang, Cui, Wang, You, Ding, Liu, Sun, and Peng]{yuan2025llms}
Lifan Yuan, Weize Chen, Yuchen Zhang, Ganqu Cui, Hanbin Wang, Ziming You, Ning Ding, Zhiyuan Liu, Maosong Sun, and Hao Peng.
\newblock From $f(x)$ and $g(x)$ to $f(g(x))$: {LLMs} learn new skills in {RL} by composing old ones.
\newblock https://husky-morocco-f72.notion.site/From-f-x-and-g-x-to-f-g-x-LLMs-Learn-New-Skills-in-RL-by-Composing-Old-Ones-2499aba4486f802c8108e76a12af3020, 2025.
\newblock Notion blog post, available online.

\bibitem[Yue et~al.(2025{\natexlab{a}})Yue, Chen, Lu, Zhao, Wang, Yue, Song, and Huang]{yue2025does}
Yang Yue, Zhiqi Chen, Rui Lu, Andrew Zhao, Zhaokai Wang, Yang Yue, Shiji Song, and Gao Huang.
\newblock Does reinforcement learning really incentivize reasoning capacity in {LLM}s beyond the base model?
\newblock In \emph{2nd AI for Math Workshop @ ICML 2025}, 2025{\natexlab{a}}.

\bibitem[Yue et~al.(2025{\natexlab{b}})Yue, Yuan, Yu, Zuo, Zhu, Xu, Chen, Wang, Fan, Du, et~al.]{yue2025vapo}
Yu~Yue, Yufeng Yuan, Qiying Yu, Xiaochen Zuo, Ruofei Zhu, Wenyuan Xu, Jiaze Chen, Chengyi Wang, TianTian Fan, Zhengyin Du, et~al.
\newblock Vapo: Efficient and reliable reinforcement learning for advanced reasoning tasks.
\newblock \emph{arXiv preprint arXiv:2504.05118}, 2025{\natexlab{b}}.

\bibitem[Zeng et~al.(2025)Zeng, Huang, Liu, Liu, He, Ma, and He]{zeng2025simplerl}
Weihao Zeng, Yuzhen Huang, Qian Liu, Wei Liu, Keqing He, Zejun Ma, and Junxian He.
\newblock Simplerl-zoo: Investigating and taming zero reinforcement learning for open base models in the wild.
\newblock \emph{arXiv preprint arXiv:2503.18892}, 2025.

\bibitem[Zhang et~al.(2025{\natexlab{a}})Zhang, Shen, Zhao, Zhang, Xu, Dou, and Bian]{zhang2025policy}
Chuheng Zhang, Wei Shen, Li~Zhao, Xuyun Zhang, Xiaolong Xu, Wanchun Dou, and Jiang Bian.
\newblock Policy filtration for {RLHF} to mitigate noise in reward models.
\newblock In \emph{Forty-second International Conference on Machine Learning}, 2025{\natexlab{a}}.

\bibitem[Zhang et~al.(2025{\natexlab{b}})Zhang, Hong, Bao, Jiang, Song, Hong, and Xiong]{gvpo}
Kaichen Zhang, Yuzhong Hong, Junwei Bao, Hongfei Jiang, Yang Song, Dingqian Hong, and Hui Xiong.
\newblock Gvpo: Group variance policy optimization for large language model post-training.
\newblock \emph{arXiv preprint arXiv:2504.19599}, 2025{\natexlab{b}}.

\bibitem[Zhao et~al.(2025{\natexlab{a}})Zhao, Meterez, Kakade, Pehlevan, Jelassi, and Malach]{zhao2025echo}
Rosie Zhao, Alexandru Meterez, Sham~M. Kakade, Cengiz Pehlevan, Samy Jelassi, and Eran Malach.
\newblock Echo chamber: {RL} post-training amplifies behaviors learned in pretraining.
\newblock In \emph{Second Conference on Language Modeling}, 2025{\natexlab{a}}.

\bibitem[Zhao et~al.(2025{\natexlab{b}})Zhao, Kang, Feng, Levine, and Song]{zhao2025learning}
Xuandong Zhao, Zhewei Kang, Aosong Feng, Sergey Levine, and Dawn Song.
\newblock Learning to reason without external rewards.
\newblock \emph{arXiv preprint arXiv:2505.19590}, 2025{\natexlab{b}}.

\bibitem[Zheng et~al.(2023)Zheng, Ke, Zhang, and Huang]{zheng-etal-2023-click}
Chujie Zheng, Pei Ke, Zheng Zhang, and Minlie Huang.
\newblock Click: Controllable text generation with sequence likelihood contrastive learning.
\newblock In \emph{Findings of the Association for Computational Linguistics: ACL 2023}, pages 1022--1040, 2023.

\bibitem[Zheng et~al.(2025)Zheng, Liu, Li, Chen, Yu, Gao, Dang, Liu, Men, Yang, et~al.]{gspo}
Chujie Zheng, Shixuan Liu, Mingze Li, Xiong-Hui Chen, Bowen Yu, Chang Gao, Kai Dang, Yuqiong Liu, Rui Men, An~Yang, et~al.
\newblock Group sequence policy optimization.
\newblock \emph{arXiv preprint arXiv:2507.18071}, 2025.

\bibitem[Zhu et~al.(2025{\natexlab{a}})Zhu, Hao, Hu, Jiao, Russell, and Tian]{zhu2025reasoning}
Hanlin Zhu, Shibo Hao, Zhiting Hu, Jiantao Jiao, Stuart Russell, and Yuandong Tian.
\newblock Reasoning by superposition: A theoretical perspective on chain of continuous thought.
\newblock \emph{arXiv preprint arXiv:2505.12514}, 2025{\natexlab{a}}.

\bibitem[Zhu et~al.(2025{\natexlab{b}})Zhu, Xia, Wei, Chen, Chen, and Meng]{zhu2025surprising}
Xinyu Zhu, Mengzhou Xia, Zhepei Wei, Wei-Lin Chen, Danqi Chen, and Yu~Meng.
\newblock The surprising effectiveness of negative reinforcement in llm reasoning.
\newblock \emph{arXiv preprint arXiv:2506.01347}, 2025{\natexlab{b}}.

\end{thebibliography}

\clearpage
\appendix

\section{More Related Works}\label{app:works}
Here, we discuss more related works to supplement the main text.

\textbf{Reinforcement Learning for LLM Reasoning}. Large language models (LLMs) are often post-trained using reinforcement learning (RL), both for preference alignment \citep{ouyang2022traininglanguagemodelsfollow,bai2022training} and to improve reasoning capabilities \citep{deepseekmath,guo2025deepseek}. Inspired by \citet{deepseekmath, liu2025intro} and \citet{swamy2025all}, this work reformulates methods like SFT, RFT \citep{trung-etal-2024-reft}, DPO, PPO, and GRPO as maximum likelihood estimation governed by a \textit{Gradient Coefficient}. This coefficient fundamentally operates by amplifying gradients for favored responses and suppressing others, with its magnitude modulating the preference intensity. Thus, the core challenge in policy gradient methods reduces to the accurate estimation of this \textit{Gradient Coefficient} (i.e., the advantage and importance ratio). For instance, AAPO \citep{xiong2025aapo} redefines advantage estimation by incorporating advantage momentum. GAPO \citep{gapo}, GVPO \citep{gvpo} and $\Delta L$ Normalization \citep{he2025deltalnormalizationrethink} employ gradient normalization to adaptively rescale each objective's gradients, thereby finding a low-variance estimator. Meanwhile, \citet{zhao2025learning} and \citet{qian2025demystifying} utilize a model’s own internal confidence measure (or entropy)—termed self-certainty to improve reasoning skills. Additionally, hybrid approaches that integrate RL with SFT on external demonstration data have been actively explored \citep{cai2025backtrack, ma2025cant, LUFFY, TAPO, rlplus}. Despite these empirical advances, the fundamental question of whether RLVR expands \citep{liu2025prorl,wen2025reinforcement,liu2025scaling,wu2025invisible,yuan2025llms,wang2025emergent,bai2025online} or shrinks \citep{yue2025does,zhao2025echo,song2025mind,dang2025assessing,he2025rewarding,ma2025learning,shah2025rethinking,gandhi2025cognitive,dang2025weight,o'mahony2024attributing} the reasoning capacities of LLMs remains an open and actively debated issue. \textit{This is precisely what we aim to uncover}.

\textbf{LLM Learning Dynamics}. Deep neural networks learn by adjusting their parameters through gradient descent. This process, known as learning dynamics, connects how model predictions change to the gradients from individual training examples. \textit{Learning dynamics prioritizes the analysis of a model's relative training behavior over its convergence, providing a means to assess the quality of individual training samples}. To name a few, \citet{pruthi2020estimating} introduce ``TracIn", a metric that measures how much a training example affects a model's predictions, \citet{xia2024less} later use it to identify the most influential examples during instruction fine-tuning of LLMs. In a similar vein, \citet{guo2024lpntk} propose a method based on the neural tangent kernel (NTK) regime to estimate the relative difficulty among different training samples. Furthermore, \citet{renlearning} highlight a unique ``squeezing effect'' to explain a previously observed phenomenon in off-policy direct preference optimization (DPO \citep{dpo}), where running DPO for too long makes even the desired outputs less likely. Since RLVR methods—exemplified by PPO and GRPO—are on-policy and dynamically evolving, we argue that analyzing learning dynamics can naturally offer a novel perspective for understanding the hot debate (capability boundary shrinkage or expansion) in RLVR.

\textbf{Gradient Analysis in Preference Optimization}. DPO \citep{dpo} has proven highly effective, as it relies solely on an offline dataset of paired preference data. However, this reliance on paired data restricts its applicability in settings where only unpaired feedback (e.g., solely positive or negative responses) is available. In response, \citet{abdolmaleki2025learning} introduce a decoupled approach that independently controls the influence of positive and negative signals, enabling learning even when only a single feedback type is available. Regarding online update methods, RAFT++ \citep{dong2023raft, xiong2025minimalist}—a simple rejection sampling approach utilizing only positively rewarded data—has been shown to deliver performance competitive with GRPO. Conversely, \citet{zhu2025surprising} report the surprising effectiveness of training exclusively on negatively rewarded samples using REINFORCE \citep{REINFORCE}, without reinforcing correct responses. \textit{As we demonstrate in the main text, the set of samples considered ``negative" is not static but evolves dynamically throughout optimization}. It imperative to analyze the underlying learning dynamics. In addition, \citet{yang2025unearthinggemsstonespolicy} and \citet{chen2025spectral} find that negative responses hold learning value (e.g., self-reflection). However, existing methods overlook this by either discarding them (RFT) or applying uniform penalties (RL), failing to leverage these nuanced signals. There are also some token-level gradient analyses: \citet{yang2025letlowoverdominate} identify that RL training is skewed by low-probability tokens' excessive gradient magnitudes, impeding the learning from essential high-probability tokens; \citet{deng2025effectnegativegradientgroup} empirically observe that GRPO can suffer from what we call Lazy Likelihood Displacement: a failure to sufficiently increase, or even a decrease in, the likelihood of correct answers during training. \textit{The above motivates us to analyze the expected update in RLVR, once again emphasizing the essential role of fine-grained probability mass allocation}.

\newpage
\section*{LLM Usage}
Regarding the use of LLMs, they were employed solely for language polishing purposes and played no role in research ideation, literature retrieval, or any other academically substantive activities.
\section{Omitted Proofs and Additional Results}
\subsection{Proof of Equation 2}\label{app:proof_eq2}
\begin{proof}
We begin by reviewing the objective function of GRPO below.
\begin{align*}
     \mathcal{J}_{\text{GRPO}}(\theta) =& \mathbb{E}_{\mathbf{x} \sim \mathcal{D},\ \{\mathbf{y}_i\}_{i=1}^G \sim \pi_{\theta_{\text{old}}}(\cdot\mid\mathbf{x})} \\
&\left[
\frac{1}{G} \sum_{i=1}^{G} \frac{1}{|\mathbf{y}_i|} \sum_{t=1}^{|\mathbf{y}_i|}\left \{ \min\left(
w_{i,t}(\theta) \hat{A}_{i,t},
\text{clip}\big(w_{i,t}(\theta), 1 - \epsilon, 1 + \epsilon\big) \hat{A}_{i,t}
\right)
\right \}-\beta\mathbb{D}_{\text{KL}}(\pi_{\theta}\mid\mid\pi_{\text{ref}})\right],  
\end{align*}
where $w_{i,t}(\theta) = \frac{\pi_{\theta}(y_{i,t} \mid \mathbf{x}, \mathbf{y}_{i,<t})}{\pi_{\theta_{\text{old}}}(y_{i,t} \mid \mathbf{x}, \mathbf{y}_{i,<t})}, \mathbb{D}_{\text{KL}}(\pi_{\theta}\mid\mid\pi_{\text{ref}})=\frac{\pi_{\text{ref}}(y_{i,t} \mid \mathbf{x}, \mathbf{y}_{i,<t})}{\pi_{\theta}(y_{i,t} \mid \mathbf{x}, \mathbf{y}_{i,<t})}-\log \frac{\pi_{\text{ref}}(y_{i,t} \mid \mathbf{x}, \mathbf{y}_{i,<t})}{\pi_{\theta}(y_{i,t} \mid \mathbf{x}, \mathbf{y}_{i,<t})}-1$, $\beta$ is the coefficient.

To better understand the model’s learning dynamics under this binary outcome reward setting, we
omit the regularization components (e.g., KL term \& clipping operation):
\begin{align*}
     \mathcal{J}_{\text{GRPO}}(\theta) =& \mathbb{E}_{\mathbf{x} \sim \mathcal{D},\ \{\mathbf{y}_i\}_{i=1}^G \sim \pi_{\theta_{\text{old}}}(\cdot\mid\mathbf{x})} \left [\frac{1}{G} \sum_{i=1}^{G} \frac{1}{|\mathbf{y}_i|} \sum_{t=1}^{|\mathbf{y}_i|}
w_{i,t}(\theta) \hat{A}_{i,t}\right ],  
\end{align*}

\begin{align*}\label{eq:main2}
 \nabla_{\theta}\mathcal{J}_{\text{GRPO}}(\theta) \ &= \mathbb{E}_{\mathbf{x} \sim \mathcal{D},\ \{\mathbf{y}_i\}_{i=1}^G \sim \pi_{\theta_{\text{old}}}(\cdot\mid\mathbf{x})} \left [\frac{1}{G} \sum_{i=1}^{G} \frac{1}{|\mathbf{y}_i|} \sum_{t=1}^{|\mathbf{y}_i|}
\nabla_{\theta}w_{i,t}(\theta) \hat{A}_{i,t}\right ] \\
&=\mathbb{E}_{\mathbf{x} \sim \mathcal{D},\ \{\mathbf{y}_i\}_{i=1}^G \sim \pi_{\theta_{\text{old}}}(\cdot\mid\mathbf{x})} \left [\frac{1}{G} \sum_{i=1}^{G} \frac{1}{|\mathbf{y}_i|} \sum_{t=1}^{|\mathbf{y}_i|}
\frac{\nabla_{\theta}\pi_{\theta}(y_{i,t} \mid \mathbf{x}, \mathbf{y}_{i,<t})}{\pi_{\theta_{\text{old}}}(y_{i,t} \mid \mathbf{x}, \mathbf{y}_{i,<t})} \hat{A}_{i,t}\right ] \\
 &=\mathbb{E}_{\mathbf{x} \sim \mathcal{D},\ \{\mathbf{y}_i\}_{i=1}^G \sim \pi_{\theta_{\text{old}}}(\cdot\mid\mathbf{x})} \left [\frac{1}{G} \sum_{i=1}^{G} \frac{1}{|\mathbf{y}_i|} \sum_{t=1}^{|\mathbf{y}_i|}
\frac{\pi_{\theta}(y_{i,t} \mid \mathbf{x}, \mathbf{y}_{i,<t})}{\pi_{\theta_{\text{old}}}(y_{i,t} \mid \mathbf{x}, \mathbf{y}_{i,<t})} \hat{A}_{i,t}\nabla_{\theta}\log\pi_{\theta}(y_{i,t} \mid \mathbf{x}, \mathbf{y}_{i,<t})\right ] \\
 \\ &=\mathbb{E}_{\mathbf{x} \sim \mathcal{D},\ \{\mathbf{y}_i\}_{i=1}^G \sim \pi_{\theta_{\text{old}}}(\cdot\mid\mathbf{x})} \left[ \frac{1}{G} \sum_{i=1}^{G} \frac{1}{|\mathbf{y}_i|} \sum_{t=1}^{|\mathbf{y}_i|} 
\underbrace{w_{i,t}(\theta)\hat{A}_{i,t}}_{\text{coefficient }} 
\nabla_{\theta} \log \pi_{\theta}(y_{i,t} \mid \mathbf{x}, \mathbf{y}_{i,<t}) \right].
\end{align*}
We complete the proof of Equation \ref{eq:main2}. Notice that $w_{i,t}$ does not affect the sign of $\hat{A}_{i,t}$.

Besides, one can also consider the gradient of the KL term (denote $\pi(y_{i,t} \mid \mathbf{x}, \mathbf{y}_{i,<t})$ as $\pi(y_{i,t})$):
\begin{align*}
 \nabla_{\theta}\beta\mathbb{D}_{\text{KL}}(\pi_{\theta}\mid\mid\pi_{\text{ref}})&= \beta\nabla_{\theta}\frac{\pi_{\text{ref}}(y_{i,t})}{\pi_{\theta}(y_{i,t})}-\beta\nabla_{\theta}\log \frac{\pi_{\text{ref}}(y_{i,t})}{\pi_{\theta}(y_{i,t})} \\
 &= -\beta\frac{\pi_{\text{ref}}(y_{i,t})}{\pi^2_{\theta}(y_{i,t})}\nabla_{\theta}\pi_{\theta}(y_{i,t})+\beta\nabla_{\theta}\log \pi_{\theta}(y_{i,t}) \\
 &= -\left[\beta\frac{\pi_{\text{ref}}(y_{i,t})}{\pi_{\theta}(y_{i,t})}-\beta\right]\nabla_{\theta}\log \pi_{\theta}(y_{i,t}).
\end{align*}
\end{proof}
\newpage
\subsection{Proof of Lemma 1}\label{app:proof_lemma1}
\begin{proof}
 Re-stating the Lemma \ref{main:lemma1}, the output of a model is the logits $\mathbf{z}=[z_1,...,z_V]^T$, which corresponds to a finite (size $V$) vocabulary set $\mathcal{V}=\{v_1,...,v_V\}$. The policy probability of the corresponding action (token) is calculated by:   $\pi(v) = \text{Softmax}(\mathbf{z})_v=\exp{(z_v)}/\sum^V_{v^{'}}\exp{(z_{v^{'}})}$.

 That is, $\log\pi(v)=\log(\exp (z_v))-\log(\sum_{v^{'}}^V\exp (z_v^{'}))=z_v-\log(\sum_{v^{'}}^V\exp (z_v^{'})).$

 Thus, for the currently sampled token $v$, let $z_v$ be its corresponding logit, we will have:
 $$
 \frac{\partial \log\pi(v)}{\partial z_v} = 1-\pi(v),
 $$
 for other unsampled tokens $u\neq v$ (with its logit $z_u$):
 $$
 \frac{\partial \log\pi(v)}{\partial z_u} = -\pi(u).
 $$
 Apply those to the gradient $\nabla_{\mathbf{z}}\mathcal{J}=\hat{A}(v)\nabla_{\mathbf{z}}\log \pi(v)$, we complete the proof of Lemma \ref{main:lemma1}.
\end{proof}

\subsection{Proposition 1 and Proof}\label{app:proof_p1}
\begin{proposition}\label{main:proposition1} Let the conditions specified in Lemma \ref{main:lemma1} hold, and denote $\Delta \mathbf{z}(\mathbf{x})=[\Delta z_1,...,\Delta z_V]^T$, the $l$-th step probability mass dynamics decompose as: 
 \begin{equation*}
     \Delta \log \pi^l(\mathbf{y}\mid \mathbf{x})=\left[I-\mathbf{e}(\pi^l(\mathbf{y}\mid \mathbf{x}))^T\right]\left[(\nabla_{\theta} \mathbf{z}^{\theta^l}(\mathbf{x}))(\nabla_{\theta} \mathbf{z}^{\theta^l}(\mathbf{x}))^T\right]\Delta \mathbf{z}^{l}(\mathbf{x})+\mathcal{O}(\eta^2\left\|\nabla_{\theta} \mathbf{z}^{\theta^l}(\mathbf{x})\right\|^2),
 \end{equation*}
 where $I$ is the identity matrix and $\mathbf{e}=[1,1,...,1]^T$, $\left[(\nabla_{\theta} \mathbf{z}^{\theta^l}(\mathbf{x}))(\nabla_{\theta} \mathbf{z}^{\theta^l}(\mathbf{x}))^T\right] \in \mathbb{R}^{V\times V}$ is the empirical neural tangent kernel, $\Delta\log \pi^l(\mathbf{y}\mid \mathbf{x})\in \mathbb{R}^{V\times 1}$. $\Delta \mathbf{z}(\mathbf{x})=\eta \nabla_{\mathbf{z}} \mathcal{J}\in \mathbb{R}^{V\times 1}$, which mainly determines the \textbf{direction and magnitude} of the policy update.
\end{proposition} 
\begin{proof}
Recall the log probabilities change in Eq. (4):
\begin{align*}
 \Delta \log \pi^l(\mathbf{y}\mid \mathbf{x}) &\triangleq \log \pi_{\theta^{l+1}}(\mathbf{y}\mid \mathbf{x}) - \log \pi_{\theta^{l}}(\mathbf{y}\mid \mathbf{x}) :=   \log \pi^{l+1}(\mathbf{y}\mid \mathbf{x}) - \log \pi^{l}(\mathbf{y}\mid \mathbf{x}) ,
\end{align*}
and we follow \citet{renlearning} using Taylor expansion to approximate $\log \pi^{l+1}(\mathbf{y}\mid \mathbf{x})$:
\begin{align*}
 \log \pi^{l+1}(\mathbf{y}\mid \mathbf{x}) = \log \pi^{l}(\mathbf{y}\mid \mathbf{x})  + \langle\nabla \log \pi^{l}(\mathbf{y}\mid \mathbf{x}),\theta^{l+1}-\theta^l\rangle+\mathcal{O}(\left\|\theta^{l+1}-\theta^l\right \|^2).
\end{align*}
Then, supposing the parameters' are updated by policy gradient, we will have (the model parameters $\theta \in \mathbb{R}^{d\times1}$):
\begin{align*}
     \Delta \log \pi^l(\mathbf{y}\mid \mathbf{x}) = \nabla_{\theta} \log \pi^{l}(\mathbf{y}\mid \mathbf{x})(\theta^{l+1}-\theta^l)+\mathcal{O}(\left\|\theta^{l+1}-\theta^l\right \|^2).
\end{align*}
Next, we use the definition of gradient and the chain rule:
\begin{align*}
\nabla_{\theta} \log \pi^{l}(\mathbf{y}\mid \mathbf{x})(\theta^{l+1}-\theta^l) &= \left[\nabla_{\mathbf{z}^{\theta^l}}\log \pi^{l}(\mathbf{y}\mid \mathbf{x})(\nabla_{\theta} \mathbf{z}^{\theta^l}(\mathbf{x}))\right] \left[\eta \nabla_{\theta^l} \mathcal{J}\right]^T \\
&= \left[\nabla_{\mathbf{z}^{\theta^l}}\log \pi^{l}(\mathbf{y}\mid \mathbf{x})(\nabla_{\theta} \mathbf{z}^{\theta^l}(\mathbf{x}))\right] \left[\eta \nabla_{\mathbf{z}^{\theta^l}} \mathcal{J}(\nabla_{\theta} \mathbf{z}^{\theta^l}(\mathbf{x}))\right]^T \\
&= \nabla_{\mathbf{z}^{\theta^l}}\log \pi^{l}(\mathbf{y}\mid \mathbf{x})\left[(\nabla_{\theta} \mathbf{z}^{\theta^l}(\mathbf{x}))(\nabla_{\theta} \mathbf{z}^{\theta^l}(\mathbf{x}))^T\right] (\eta \nabla_{\mathbf{z}^{\theta^l}} \mathcal{J}) \\
&= \left[I-\mathbf{e}(\pi^l(\mathbf{y}\mid \mathbf{x}))^T\right]\left[(\nabla_{\theta} \mathbf{z}^{\theta^l}(\mathbf{x}))(\nabla_{\theta} \mathbf{z}^{\theta^l}(\mathbf{x}))^T\right]\Delta \mathbf{z}^{l}(\mathbf{x}).
\end{align*}
For the higher-order term:
\begin{align*}
\theta^{l+1}-\theta^l = \eta \nabla_{\theta^l} \mathcal{J} = \eta (\nabla_{\theta} \mathbf{z}^{\theta^l}(\mathbf{x}))^T\nabla_{\mathbf{z}^{\theta^l}} \mathcal{J},    
\end{align*}
and from the practical application and Lemma \ref{main:lemma1}, the term $\nabla_{\mathbf{z}^{\theta^l}} \mathcal{J}$ is usually bounded, we get:
\begin{align*}
 \mathcal{O}(\left\|\theta^{l+1}-\theta^l\right \|^2)=\mathcal{O}(\eta^2\left\|\nabla_{\theta} \mathbf{z}^{\theta^l}(\mathbf{x})\right \|^2).  
\end{align*}
We complete the proof. $\left[(\nabla_{\theta} \mathbf{z}^{\theta^l}(\mathbf{x}))(\nabla_{\theta} \mathbf{z}^{\theta^l}(\mathbf{x}))^T\right] \in \mathbb{R}^{V\times V}$ denotes the empirical neural tangent kernel (NTK), which remains nearly constant throughout the training process \citep{renlearning, arora2019exact, jacot2018neural}. As a result, $\Delta \mathbf{z}^l(\mathbf{x})$ primarily governs both the direction and magnitude of the policy update.
\end{proof}
\newpage
\subsection{Proof of Theorem 1}\label{app:proof_p2}
\textbf{Theorem 1.} Under the conditions stated in Lemma \ref{main:lemma1}, we assume that $\mathbf{x}  \sim \mathcal{D}$ is i.i.d., the expected group relative policy gradient $\nabla_{\mathbf{z}} \mathcal{J} \in \mathbb{R}^{V\times1}$ is $\mathbb{E}_{\mathbf{x}\sim \mathcal{D},\{u_i\}_{i=1}^G \sim \pi(\cdot\mid \mathbf{x})} \left[\frac{1}{G}\sum_{i=1}^G\hat{A}(u_i) \nabla_{\mathbf{z}} \log \pi(u_i)\right]$. Then the expected logits update is:
\begin{align*}
    \mathbb{E}(\Delta z_v^l)=\eta\cdot\mathbb{E}_{u \sim \pi^l(\cdot\mid \mathbf{x})} \left[\hat{A}(u) \nabla_{z_v^l} \log \pi^l(u)\right]=\eta\cdot\pi^l(v)\left[(1-\pi^l(v))\hat{A}(v)-\sum_{u\neq v}\pi^l(u)\hat{A}(u)\right].
\end{align*}
\begin{proof}
 From Lemma \ref{main:lemma1}, the policy gradient of sampling a token (action) $u$ once from the policy $\pi^l(\cdot\mid \mathbf{x})$ is $\hat{A}(u)\nabla_{\mathbf{z}}\log\pi(u)$. Thus, the expected group relative policy gradient is the following:
 \begin{align*}
  \nabla_{\mathbf{z}}\mathcal{J}= \mathbb{E}_{\mathbf{x}\sim \mathcal{D},\{u_i\}_{i=1}^G \sim \pi(\cdot\mid \mathbf{x})} \left[\frac{1}{G}\sum_{i=1}^G\hat{A}(u_i) \nabla_{\mathbf{z}} \log \pi(u_i)\right].   
 \end{align*}
 Given that $\mathbf{x} \sim \mathcal{D}$ is i.i.d. and $\{u_i\}_{i=1}^G$ are randomly sampled from $\pi(\cdot\mid \mathbf{x})$, we derive an unbiased estimator:
 \begin{align*}
  \nabla_{\mathbf{z}}\mathcal{J}= \mathbb{E}_{u \sim \pi(\cdot\mid \mathbf{x})} \left[\hat{A}(u) \nabla_{\mathbf{z}} \log \pi(u)\right]=\sum_u\pi(u) \hat{A}(u) \nabla_{\mathbf{z}} \log \pi(u) \in \mathbb{R}^{V\times 1}.   
 \end{align*}
Apply those to Lemma \ref{main:lemma1}, we complete the proof.
\end{proof}
\subsection{Details of Relative Negative Gradients}\label{app:details_n}
Referring back to Eq.(\ref{eq:main1}) and Eq.(\ref{eq:main2}), taking GRPO as an example, we obtain the gradient of the objective function in the following form.
\begin{align*}
 \nabla_{\theta}\mathcal{J}_{\text{GRPO}}(\theta) \ = \mathbb{E}_{\mathbf{x},\{\mathbf{y}_i\}_{i=1}^G\sim \pi_{\theta_{\text{old}}}(\cdot\mid\mathbf{x}) } \left[ \frac{1}{G} \sum_{i=1}^{G} \frac{1}{|\mathbf{y}_i|} \sum_{t=1}^{|\mathbf{y}_i|} 
\underbrace{w_{i,t}(\theta)\hat{A}_{i,t}}_{\text{coefficient }} 
\nabla_{\theta} \log \pi_{\theta}(y_{i,t} \mid \mathbf{x}, \mathbf{y}_{i,<t}) \right].    
\end{align*}
Since the advantage $\hat{A}_{i,t}$ is estimated from the currently sampled group $i=1,\cdots,G$ each time, we refer to it as the relative advantage, and correspondingly, this gradient is termed the relative policy gradient. Consequently, for the relative negative gradients, we exclusively utilize gradient information where $\hat{A}_{i,t}<0$ during the gradient update:
\begin{align*}
  \nabla_{\theta}\mathcal{J}_{\text{GRPO-N}}(\theta) \ = \mathbb{E}_{\mathbf{x},\{\mathbf{y}_i\}_{i=1}^G\sim \pi_{\theta_{\text{old}}}(\cdot\mid\mathbf{x}) } \left[ \frac{1}{G} \sum_{i=1}^{G} \frac{1}{|\mathbf{y}_i|} \sum_{t=1}^{|\mathbf{y}_i|} 
\mathbb{I}(\hat{A}_{i,t})\cdot w_{i,t}(\theta)\hat{A}_{i,t}
\nabla_{\theta} \log \pi_{\theta}(y_{i,t} \mid \mathbf{x}, \mathbf{y}_{i,<t}) \right],
\end{align*}
where $\mathbb{I}(\hat{A}_{i,t})$ is an indicator variable that equals 1 if $\hat{A}_{i,t} < 0$, and 0 otherwise.
\newpage
\section{Extension to Experiments}
\textbf{Reproducibility statement}. We employed open-source algorithms and data to validate our theoretical analysis, and have reported all hyperparameter settings to facilitate reproducibility.\\
(1) open-source code: \url{https://github.com/volcengine/verl}.\\
(2) all datasets can be found in: \url{https://huggingface.co/datasets}.\\
(3) toy example details are provided in: Algorithm \ref{app:Algorithm1}.
\subsection{Algorithm for Logits Update}\label{app:Algorithm1}

\begin{algorithm}[H]
\caption{Logits Update for Softmax Parameterization: A Toy Example}
\begin{algorithmic}
\State \textbf{Require:} learning rate $\eta$, number of samples per update $G$, true rewards $r$, optimization steps $N$
\State Initialize policy parameters (logits) $\mathbf{z}$
\For{$l = 1$ to $N$}
    \State Compute current policy $\pi \gets \text{Softmax}(\mathbf{z})$
    \State Sample $G$ actions from policy $\pi$: $\{a_1, a_2, ..., a_G\}$
    \State Estimate advantage $\hat{A}[a_i] = r[a_i] - \text{mean}(\{r[a_j]\}_{j=1}^G)$ 
    \State Initialize relative policy gradient $\mathbf{g} \gets \mathbf{0}$
    \For{each sampled action $a_i$ where $i = 1$ to $G$}
        \State $\mathbf{g}[a_i] \gets \mathbf{g}[a_i] + (1 - \pi[a_i]) \cdot \hat{A}[a_i]$
        \For{each other action $a_j \neq a_i$}
            \State $\mathbf{g}[a_j] \gets \mathbf{g}[a_j] - \pi[a_j] \cdot \hat{A}[a_i]$
        \EndFor
    \EndFor
    \State Apply Adam update: $\mathbf{z} \gets \mathbf{z} + \eta \cdot \mathbf{g} / G$
\EndFor
\State \Return Optimized policy parameters $\mathbf{z}$
\end{algorithmic}\label{alg1}    
\end{algorithm}

\subsection{Hyperparameter Settings}\label{app:sub_settings}
Our experimental configuration follows that of \citet{zhu2025surprising}.

\textbf{Training setup}.
The prompt batch size is set to 1,024, with 8 rollouts generated per prompt. During training, the sampling temperature is set to 1.0. The maximum context length is configured as 4,096 tokens for both \texttt{Qwen2.5-Math-7B} and \texttt{Llama-3.2-3B-Instruct}. Model updates are performed with a mini-batch size of 256 and a learning rate of $1\times 10^{-6}$. For all algorithms, a KL penalty term is incorporated into the final loss function, using a coefficient of $1\times 10^{-3}$. The clip ratio is set to 0.2. Additionally, an entropy bonus is applied to all objectives with a coefficient of $1\times 10^{-4}$. All experiments are conducted on a single node with 4 NVIDIA A100 GPUs.

\textbf{Evaluation setup}. During evaluation, we sample 256 responses per prompt for both \texttt{Qwen2.5-Math-7B} and \texttt{Llama-3.2-3B-Instruct} using a temperature of 0.6 and a top‑$p$ value of 0.95. Since the test sets of ARC-c (1,170) and MMLU-Pro (12,000) are relatively large, and sampling 256 times requires substantial computation time, we randomly selected 128 questions and repeated the test three times to obtain the average.

\textbf{Prompt template}. Our primary objective is to validate theoretical findings; therefore, a uniform prompt \citep{zeng2025simplerl} was sampled for all models:
\begin{lstlisting}[style=PythonStyle]
<|im_start|>system
You are a helpful assistant.<|im_end|>
<|im_start|>user
{input}
Please reason step by step, and put your final answer within \boxed{}.
<|im_end|>
<|im_start|>assistant
\end{lstlisting}

\newpage
\subsection{More Evaluation Results}\label{app:more evaluation}
\textbf{Training dynamics of LLama-3.2-3B-Instruct}.
\begin{figure}[ht]
  \centering
    \begin{subfigure}[b]{0.2445\textwidth}
        \centering
        \includegraphics[height=3.42cm]{   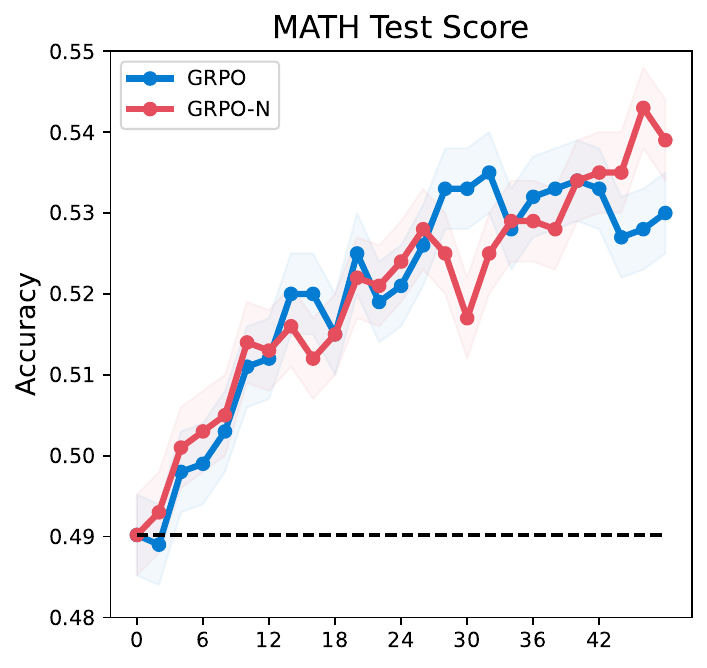}
    \end{subfigure}
    \begin{subfigure}[b]{0.2445\textwidth}
        \centering
        \includegraphics[height=3.42cm]{   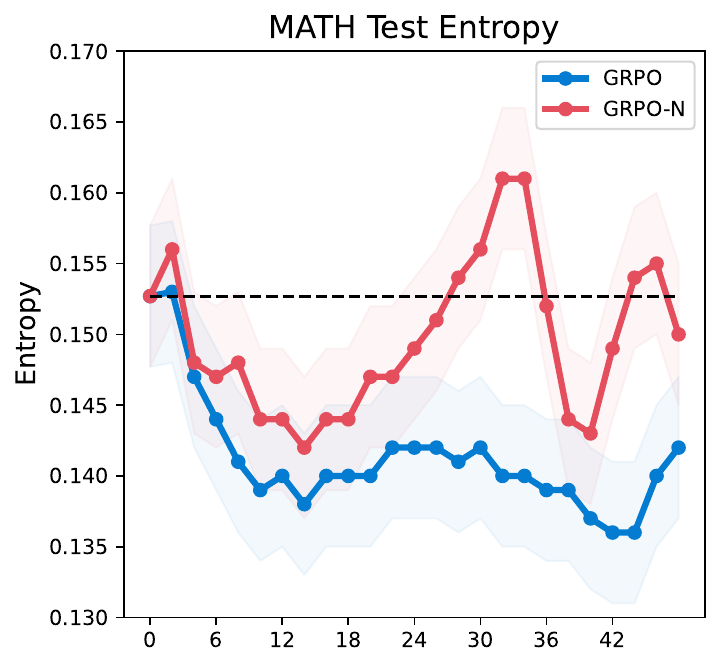}
    \end{subfigure}
    \begin{subfigure}[b]{0.2445\textwidth}
        \centering
        \includegraphics[height=3.42cm]{   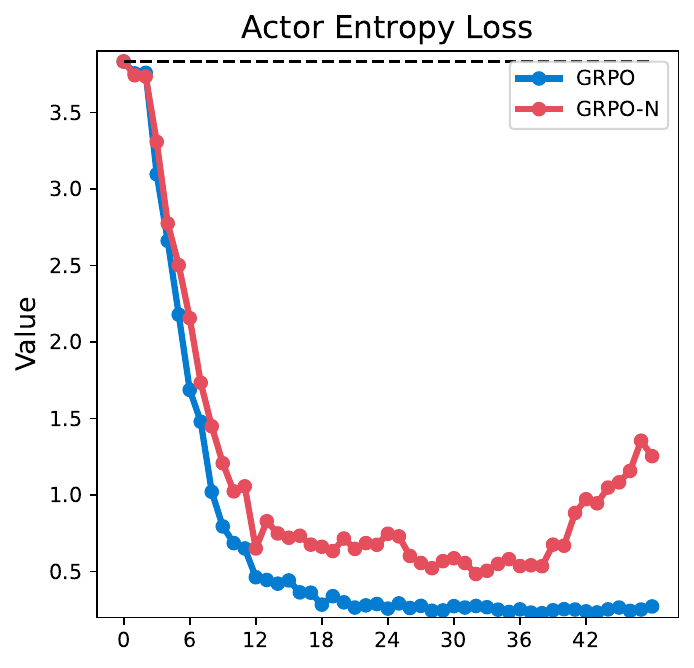}
    \end{subfigure}
    \begin{subfigure}[b]{0.2445\textwidth}
        \centering
        \includegraphics[height=3.42cm]{   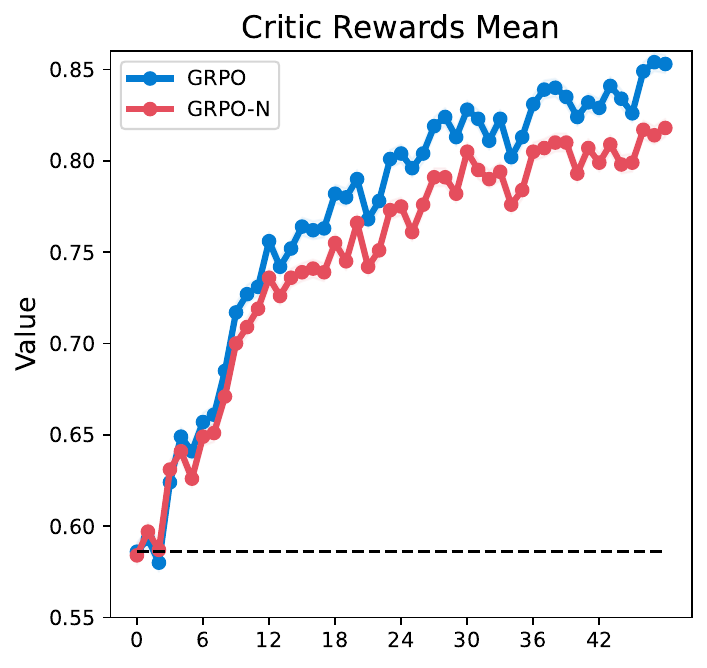}
    \end{subfigure} 
  \caption{Comparison of the training dynamics of GRPO, GRPO-N on the MATH benchmark across training steps, using the \texttt{LLama-3.2-3B-Instruct} model with a prompt batch size of 1,024. \textbf{Left Part:} (Left) the greedy decoding accuracy on the MATH test set and (Center Left) the model's entropy on the MATH test set. \textbf{Right Part:} (Center Right) the actor entropy loss and (Right) critic rewards mean during training.}
  \label{app:figure3}
\end{figure}

\begin{table*}[htb]
\centering
\setlength{\tabcolsep}{8pt}
\small
\caption{Pass@$k$ of \texttt{Llama-3.2-3B-Instruct} on AMC 2023, AIME 2024, AIME 2025. For each $k$, \textbf{bold} and \underline{underlined} numbers indicate the best and second-best results, respectively.}
\begin{tabular}{lccccccccc}
\toprule
\textbf{Algorithm} & \multicolumn{9}{c}{\textbf{Pass@$k$}} \\
$k$ & 1 & 2 & 4 & 8 & 16 & 32 & 64 & 128 & 256 \\
\midrule
\textbf{} & \multicolumn{9}{c}{\textbf{AMC 2023}} \\
Base Model & 23.4 & 34.3 & 47.7 & \textbf{61.7} & \textbf{74.4} & \textbf{84.7} & \textbf{92.1} & \textbf{96.8} & \textbf{100.0} \\
GRPO & \textbf{31.1} & \textbf{41.7} & \underline{51.3} & 58.7 & 64.7 & 70.7 & 76.9 & 83.0 & 87.5 \\
    \rowcolor{gray!10} GRPO-N & \underline{30.3} & \underline{41.6} & \textbf{52.4} & \underline{60.8} & \underline{67.5} & \underline{74.2} & \underline{81.0} & \underline{87.4} & \underline{92.5} \\
\midrule
\textbf{} & \multicolumn{9}{c}{\textbf{AIME 2024}} \\
Base Model & 6.9 & 11.5 & 17.5 & 23.8 & 29.4 & 33.7 & \textbf{37.5} & \textbf{42.7} & \textbf{50.0} \\
GRPO & \underline{15.7} & \underline{20.6} & \underline{25.1} & \underline{29.1} & \underline{32.2} & \underline{34.4} & 36.1 & 37.9 & 40.0 \\
\rowcolor{gray!10} GRPO-N & \textbf{16.2} & \textbf{21.2} & \textbf{25.8} & \textbf{29.9} & \textbf{33.2} & \textbf{35.2} & \underline{37.3} & \underline{40.8} & \underline{46.7} \\
\midrule
\textbf{} & \multicolumn{9}{c}{\textbf{AIME 2025}} \\
Base Model & 0.4 & 0.9 & 1.7 & 3.2 & 5.6 & \underline{9.2} & \underline{14.6} & \textbf{23.2} & \textbf{36.7} \\
GRPO & \textbf{0.6} & \textbf{1.1} & \textbf{2.1} & \underline{3.8} & \underline{6.2} & 9.0 & 11.7 & 14.4 & 16.7 \\
\rowcolor{gray!10} GRPO-N & \underline{0.5} & \underline{1.0} & \underline{2.0} & \textbf{3.8} & \textbf{6.6} & \textbf{10.7} & \textbf{15.5} & \underline{20.6} & \underline{26.7} \\
\bottomrule
\end{tabular}
\end{table*}

\begin{table*}[ht]
\centering
\setlength{\tabcolsep}{8pt}
\small
\caption{Evaluation results of \texttt{Qwen2.5-Math-7B} on MATH-500. For each $k$, \textbf{bold} and \underline{underlined} numbers indicate the best and second-best results, respectively.}
\begin{tabular}{lccccccccc}
\toprule
\textbf{Algorithm} & \multicolumn{9}{c}{\textbf{Pass@$k$}} \\
$k$ & 1 & 2 & 4 & 8 & 16 & 32 & 64 & 128 & 256 \\
\midrule
\textbf{} & \multicolumn{9}{c}{\textbf{MATH-500}} \\
Base Model & 40.7 & 51.5 & 58.9 & 64.1 & \textbf{68.5} & \textbf{72.9} & \textbf{77.9} & \textbf{83.2} & \textbf{88.0}   \\
GRPO & \underline{53.3} & \underline{57.9} & 61.2 & 63.5 & 65.3& 67.0& 68.9& 70.8&  72.6\\
\rowcolor{gray!10} GRPO-N & 53.0 & 57.9 & \underline{61.3} & \underline{63.7} & 65.3 &66.8 & 68.5 &70.2 &72.2 \\
GSPO & 53.0 & 57.7 & 61.0 & 63.4 & 65.2 &66.8 &68.6 &70.5 &72.8 \\
\rowcolor{gray!10} GSPO-N & \textbf{54.1} & \textbf{58.8} & \textbf{62.0} & \textbf{64.1} & \underline{65.9} & \underline{67.6} & \underline{69.5} & \underline{71.5} & \underline{73.4} \\
\bottomrule
\end{tabular}
\end{table*}
\newpage
\subsection{Discussion on RL Tricks}\label{app:discussion tricks}
We also review some widely adopted RL tricks, such as: increasing the number of
rollout samples, raising the training temperature, more than three actions case.

$\bullet$ \textbf{The number of rollout samples}: a larger $G$ leads to more stable optimization and
does not affect our main findings and conclusions, the two-stage dynamic of exploitation and exploration.
\begin{figure}[htb]
  \centering
    \begin{subfigure}[b]{0.245\textwidth}
        \centering
        \includegraphics[height=3.45cm]{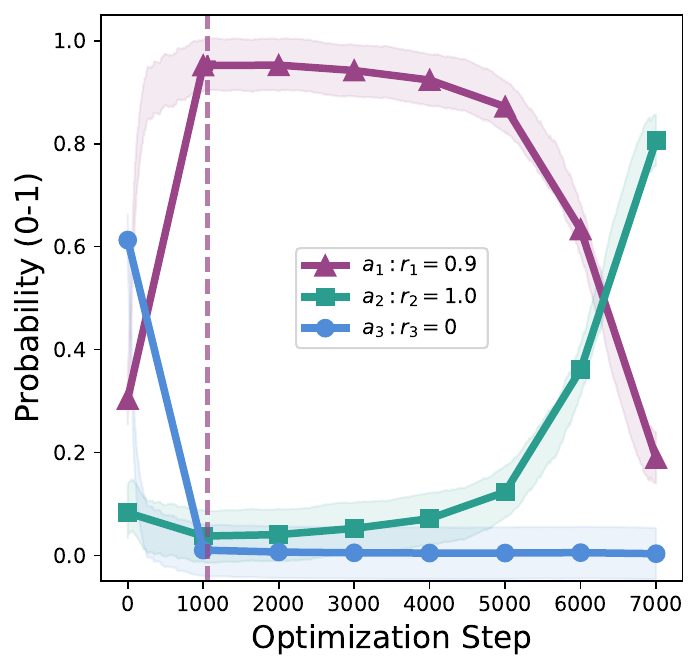}
    \end{subfigure}
    \begin{subfigure}[b]{0.245\textwidth}
        \centering
        \includegraphics[height=3.45cm]{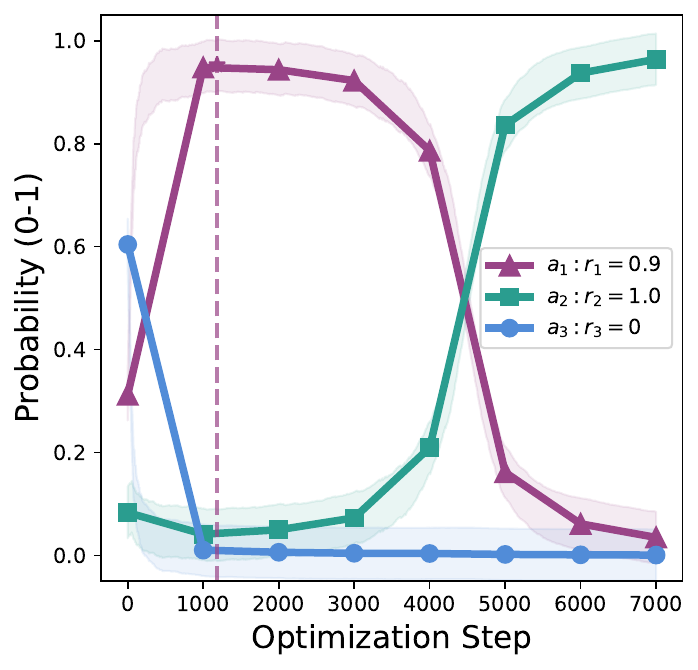}
    \end{subfigure}
    \begin{subfigure}[b]{0.245\textwidth}
        \centering
        \includegraphics[height=3.45cm]{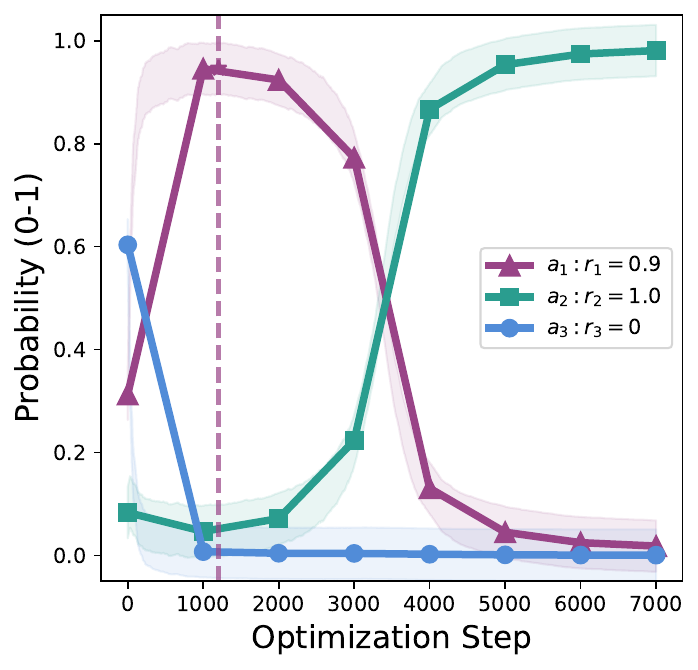}
    \end{subfigure}
    \begin{subfigure}[b]{0.245\textwidth}
        \centering
        \includegraphics[height=3.45cm]{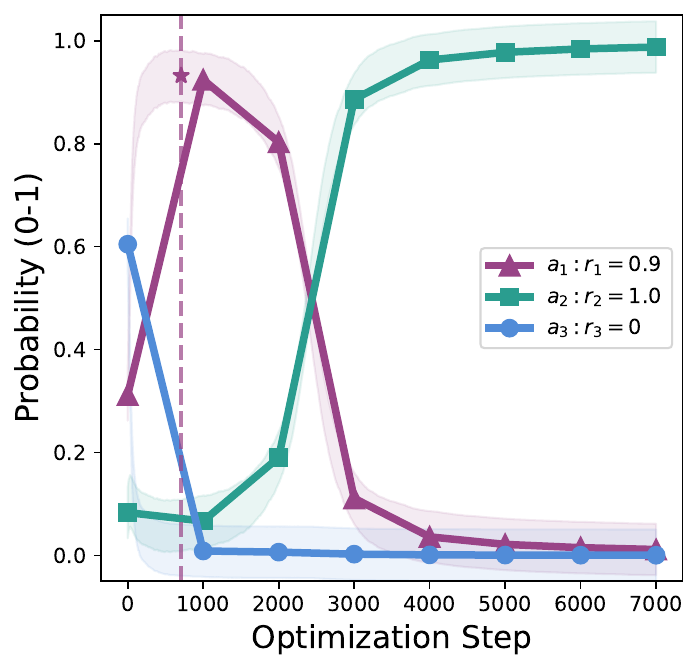}
    \end{subfigure} 
  \caption{Dynamics of the policy probability mass during optimization for different numbers of rollout samples ($[2, 3, 5, 10]$), with action rewards $r$ and initial policy probabilities $\pi$ held constant.}
    \vspace{-5mm}
\end{figure}

$\bullet$ \textbf{Raising the training temperature}: according to \citet{Polaris2025}, increasing the sampling temperature enhances the diversity of generated outcomes. Consequently, employing a higher temperature is advisable to obtain a more varied set of trajectories for model training. The default temperature value in our other experiments is $\tau=1.0$, that is $\pi(\cdot)=\text{Softmax}(\mathbf{z}/\tau)$.
\begin{figure}[htb]
  \centering
    \begin{subfigure}[b]{0.245\textwidth}
        \centering
        \includegraphics[height=3.45cm]{pics1/t_G_2_r1_0.9_r2_1.0_r1_0_1.8_0.5_2.5.pdf}
    \end{subfigure}
    \begin{subfigure}[b]{0.245\textwidth}
        \centering
        \includegraphics[height=3.45cm]{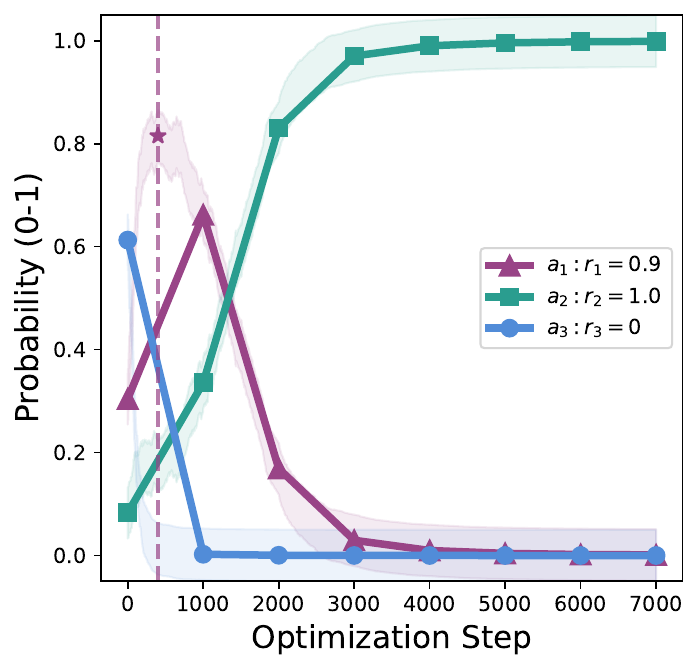}
    \end{subfigure}
    \begin{subfigure}[b]{0.245\textwidth}
        \centering
        \includegraphics[height=3.45cm]{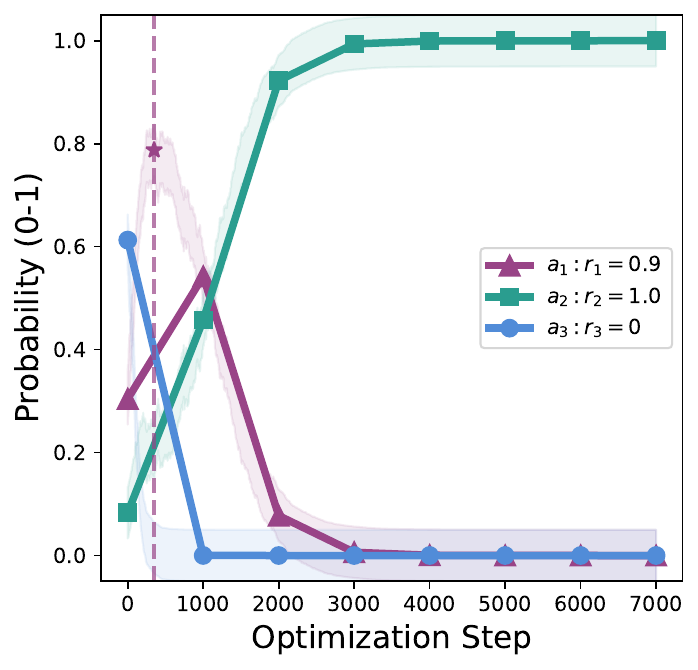}
    \end{subfigure}
    \begin{subfigure}[b]{0.245\textwidth}
        \centering
        \includegraphics[height=3.45cm]{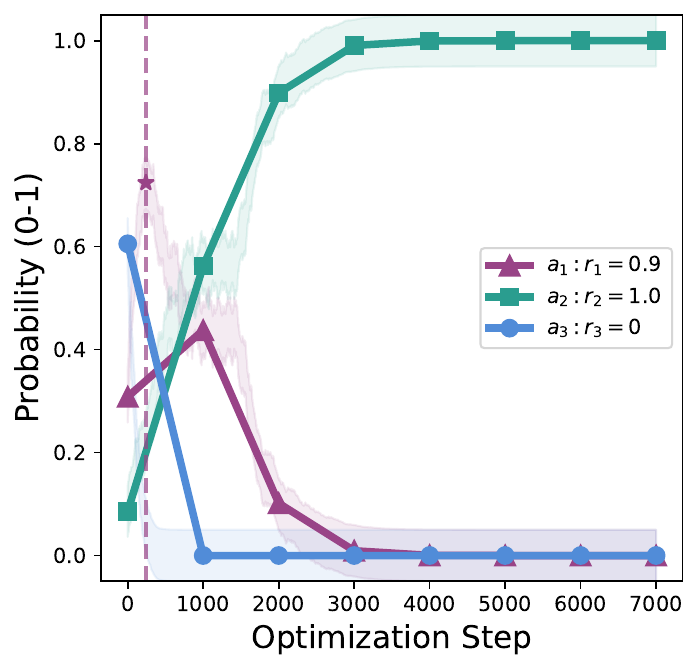}
    \end{subfigure} 
  \caption{Dynamics of the policy probability mass during optimization for different training temperature values ($[1, 2, 5, 20]$), with action rewards $r$ and initial policy probabilities $\pi$ held constant.}
    \vspace{-5mm}
\end{figure}

$\bullet$ \textbf{More than three actions case}: from Theorem \ref{main:proposition2}, we have: $\mathbb{E}(\Delta z(a_i))=\eta\pi(a_i)\left[(1-\pi(a_i))\hat{A}(a_i)-\sum_{j\neq i}^4\pi(a_j)\hat{A}(a_j)\right]$. 

Denote the action with the largest $r$ as $a_{\text{max}}$ and the action with the smallest $r$ as $a_{\text{min}}$. It can be readily shown that $\mathbb{E}(\Delta z(a_{\text{max}}))$ is always greater than or equal to 0, while $\mathbb{E}(\Delta z(a_{\text{min}}))$ is consistently less than 0. For other actions, the probabilities generally exhibit a two-stage dynamic.
\begin{figure}[htb]
  \centering
    \begin{subfigure}[b]{0.245\textwidth}
        \centering
        \includegraphics[height=3.45cm]{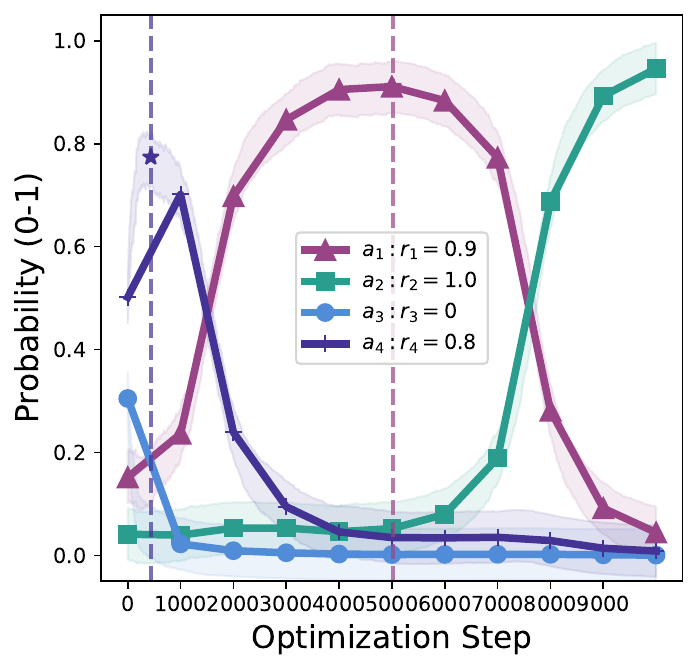}
    \end{subfigure}
    \begin{subfigure}[b]{0.245\textwidth}
        \centering
        \includegraphics[height=3.45cm]{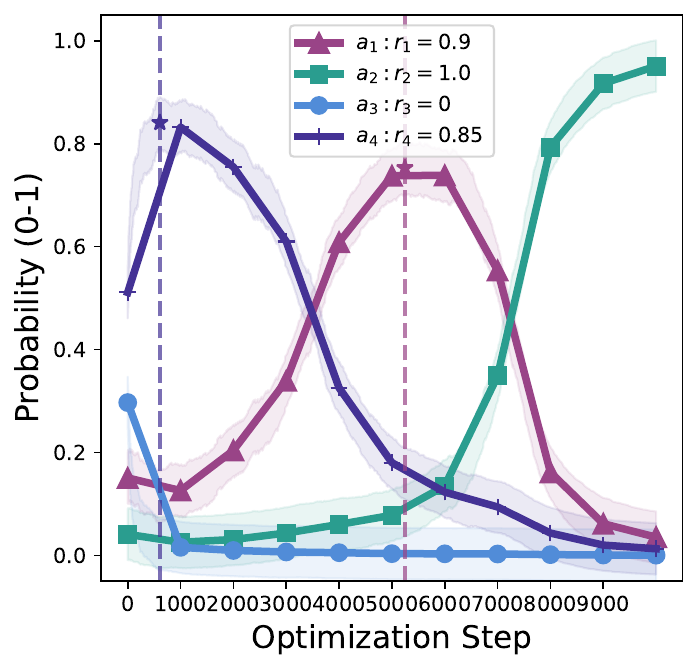}
    \end{subfigure}
    \begin{subfigure}[b]{0.245\textwidth}
        \centering
        \includegraphics[height=3.45cm]{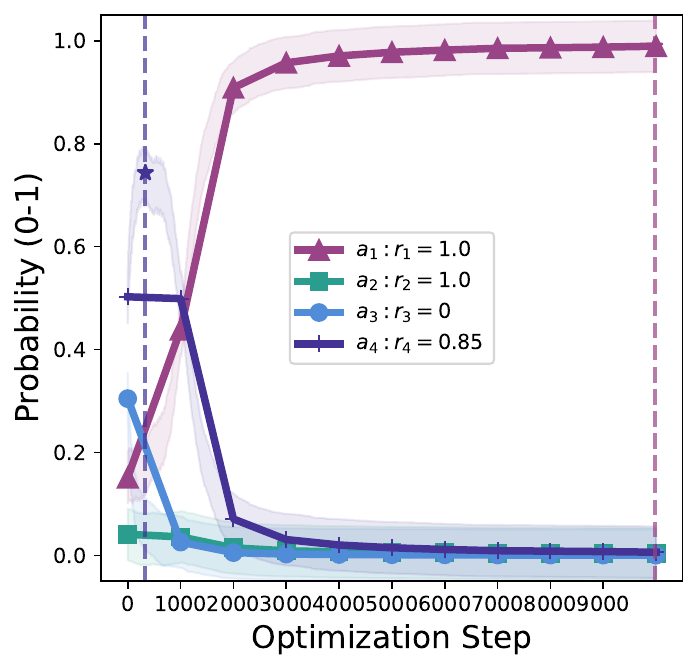}
    \end{subfigure}
    \begin{subfigure}[b]{0.245\textwidth}
        \centering
        \includegraphics[height=3.45cm]{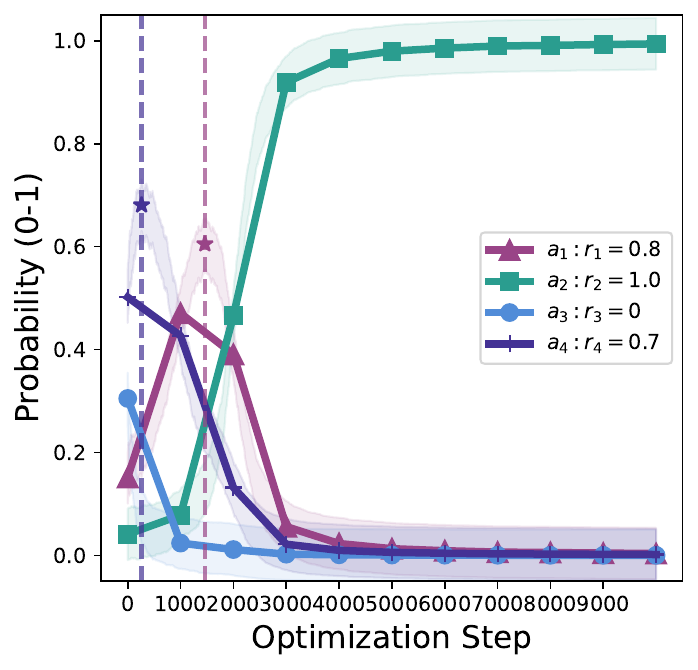}
    \end{subfigure} 
  \caption{Dynamics of the policy probability mass during optimization for action space consists of four actions, with action rewards $r$ and initial policy probabilities $\pi$ held constant.}
    \vspace{-5mm}
\end{figure}

\newpage
\subsection{Entropy Behavior Analysis from Different Levels}
\textbf{Set up.} To investigate how different algorithms reshapes the sampling distribution, we compare the base model with the RLVR trained model (using the experimental setup detailed in Section \ref{main:-n}).

Following \citet{wu2025invisible},  we quantify changes in the output distribution using two entropy metrics:

$\bullet$  Answer-Level Entropy: Let $\{o^{(1)}, \ldots, o^{(G)}\}$ represent the answers extracted from each generated sequence $\mathbf{y}_{i}$ (with NA denoting incomplete or invalid outputs), and let $\{o_{1}^{*}, \ldots, o_{M}^{*}\}$ be the set of $M$ distinct answers. Denote by $f_j$ the frequency of answer $o_j^{*}$, and define the empirical probability as $p_j = \frac{f_j}{G}$. The answer-level entropy is then defined as: $\text{AnswerEntropy} = -\sum_{j=1}^{M} p_j \log p_j$. This metric quantifies the global diversity across output completions, where lower entropy values indicate a greater degree of answer-level certainty.

$\bullet$  Token-Level Entropy: Let $\mathcal{V}$ denote the vocabulary and $\mathbf{y}_{i} = (y_{i,1}, y_{i,2}, \ldots, y_{i,T})$ denote the $i$-th generated sequence of length $T$ for $1 \leq i \leq N$. At each timestep $t$, the model outputs a probability distribution $p^{(i)}_t(v)$ over vocabulary tokens $v \in \mathcal{V}$. The entropy of this distribution is given by: $H(p^{(i)}_t) = -\sum_{v \in \mathcal{V}} p^{(i)}_t(v) \log p^{(i)}_t(v)$. The average token-level entropy over all $G$ sequences and all timesteps is then computed as: $\text{TokenEntropy} = \frac{1}{G} \frac{1}{T} \sum_{i=1}^{G} \sum_{t=1}^{T} H(p^{(i)}_t)$.

\begin{table*}[ht]
\centering
\caption{Summary of entropy metrics across math reasoning benchmarks.}
\label{tab:entropy}
\resizebox{1.0\textwidth}{!}{%
\begin{tabular}{c|c|cccccc}
\toprule
\textbf{Metric} & \textbf{Model} & \textbf{AMC 2023} & \textbf{AIME 2024}  & \textbf{AIME 2025}   \\
\midrule 
  &   \texttt{Qwen2.5-Math-7B}       & 2.563  & 4.263  & 4.904   \\
  &   GRPO       & 1.667  & 3.691  & 4.916   \\
  &   GRPO-N       & 1.741  & 3.748  & 4.957   \\
  Answer-Level &   GSPO & 1.641  & 3.484  & 4.878   \\
  Entropy &   GSPO-N & 1.696 & 3.712 & 4.943  \\ \cline{2-8}\addlinespace[3pt]
  &   \texttt{Llama-3.2-3B-Instruct}      & 3.937  & 5.333  & 6.062   \\
  &   GRPO       & 2.513  & 2.888  & 3.694   \\
  &  GRPO-N      & 2.633  & 3.042  & 4.370   \\ 
\midrule
  &   \texttt{Qwen2.5-Math-7B}       & 0.118  & -   & 0.134   \\
  &   GRPO       & 0.063  &  -  & 0.054   \\
  &   GRPO-N       & 0.124  &  -  & 0.092   \\
  Token-Level &   GSPO & 0.087  &  -  & 0.065   \\
  Entropy &   GSPO-N & 0.095 & -  & 0.075  \\ \cline{2-8}\addlinespace[3pt]
  &   \texttt{Llama-3.2-3B-Instruct}      & 0.127  &- & 0.107    \\
  &   GRPO       & 0.123  &  -  & 0.089   \\
  &  GRPO-N      & 0.105  &  -  & 0.107 & & &  \\ 
\bottomrule
\end{tabular}}
\end{table*}

\subsection{Full Procedure for Case Study}\label{app:case}
A comparison of the correct responses of GRPO and GRPO-N in mathematical reasoning task (a test case from AMC 2023). The base model is \texttt{Qwen2.5-Math-7B}, which is trained on the training set of MATH dataset.

\textbf{GRPO:} It exhibits a tendency to repeatedly make similar errors, such as consistently generating code with TypeError issues, and struggles to correct them.
\begin{lstlisting}
To solve the problem, we need to find the maximum value of the 
imaginary part of the complex number \(z\) that satisfies the
equation \(|1 + z + z^2| = 4\). Let's break down the problem step-by-step.

1. Express \(z\) in terms of its real and imaginary parts:
   Let \(z = x + yi\), where \(x\) and \(y\) are real numbers, 
   and \(i\) is the imaginary unit. 
   Then \(z^2 = (x + yi)^2 = x^2 - y^2 + 2xyi\).

2. Substitute \(z\) and \(z^2\) into the given equation:
   \[
   |1 + z + z^2| = |1 + (x + yi) + (x^2 - y^2 + 2xyi)| 
   = |(1 + x + x^2 - y^2) + (y + 2xy)i| = 4
   \]
   The magnitude of a complex number \(a + bi\) 
   is given by \(\sqrt{a^2 + b^2}\). So we have:
   \[
   \sqrt{(1 + x + x^2 - y^2)^2 + (y + 2xy)^2} = 4
   \]
   Squaring both sides, we get:
   \[
   (1 + x + x^2 - y^2)^2 + (y + 2xy)^2 = 16
   \]

3. Maximize the imaginary part \(y\):
   We need to find the maximum value of \(y\) 
   that satisfies the above equation. 
   Let's use Python and SymPy to solve this.

```python
import sympy as sp

# Define variables
x, y = sp.symbols('x y', real=True)

# Define the equation
equation = (1 + x + x**2 - y**2)**2 + (y + 2*x*y)**2 - 16

# Solve the equation for y in terms of x
solutions = sp.solve(equation, y)

# Find the maximum value of y
max_y = max([sp.simplify(sol) for sol in solutions])

# Simplify the maximum value of y
max_y_simplified = sp.simplify(max_y)

# Extract m and n from the simplified form of max_y
m, n = sp.fraction(max_y_simplified**2)

# Since m and n are relatively prime, we can directly sum them
result = m + n

print(result)
```
```output
TypeError: cannot determine truth value of Relational
```
The error occurs because SymPy's `solve` function 
can return symbolic solutions that include inequalities 
or other relational expressions. 
To handle this, we need to ensure that we are only
considering real solutions and then find the maximum value of \(y\).

Let's refine the approach by solving the equation numerically
and then finding the maximum value of \(y\).
We can use SymPy's `nsolve` function to find 
numerical solutions and then determine the maximum value of \(y\).

Here's the revised code:

```python
import sympy as sp

# Define variables
x, y = sp.symbols('x y', real=True)

# Define the equation
equation = (1 + x + x**2 - y**2)**2 + (y + 2*x*y)**2 - 16

# Find critical points by solving the equation
critical_points = sp.solve(equation, y)

# We need to find the maximum value of y
max_y = None
for point in critical_points:
    # Check if the solution is real
    if point.is_real:
        if max_y is None or point > max_y:
            max_y = point

# Simplify the maximum value of y
max_y_simplified = sp.simplify(max_y)

# Extract m and n from the simplified form of max_y
m, n = sp.fraction(max_y_simplified**2)

# Since m and n are relatively prime, we can directly sum them
result = m + n

print(result)
```
```output
TypeError: unsupported operand type(s) for ** or pow(): 'NoneType' and 'int'
```
The error indicates that the `max_y` variable is `None`,
which means that the solutions returned 
by `sp.solve` did not yield any real values for \(y\).
This suggests that we need to handle the solutions
more carefully and ensure that we are only considering real solutions.

Let's use a different approach to find the maximum value of \(y\).
We can use numerical methods to find the maximum value of \(y\)
by evaluating the solutions over a range of \(x\) values.

Here's the revised code using numerical methods:

```python
import sympy as sp
import numpy as np

# Define variables
x, y = sp.symbols('x y', real=True)

# Define the equation
equation = (1 + x + x**2 - y**2)**2 + (y + 2*x*y)**2 - 16

# Function to find the maximum value of y for a given x
def find_max_y():
    max_y = None
    for x_val in np.linspace(-10, 10, 1000):
        solutions = sp.solve(equation.subs(x, x_val), y)
        for sol in solutions:
            if sol.is_real:
                if max_y is None or sol > max_y:
                    max_y = sol
    return max_y

# Find the maximum value of y
max_y = find_max_y()

# Simplify the maximum value of y
max_y_simplified = sp.simplify(max_y)

# Extract m and n from the simplified form of max_y
m, n = sp.fraction(max_y_simplified**2)

# Since m and n are relatively prime, we can directly sum them
result = m + n

print(result)
```
```output
timeout_decorator.timeout_decorator.TimeoutError: 'Timed Out'
```
The numerical approach can be computationally intensive and
may not be the most efficient way to solve this problem.
Instead, let's use a more analytical approach to find the maximum value of \(y\).

We can use the fact that the equation \(|1 + z + z^2| = 4\) represents
a certain geometric shape in the complex plane.
To find the maximum value of the imaginary part of \(z\),
we can use the properties of complex numbers and their magnitudes.

Let's rewrite the equation in a more manageable form. We have:
\[
|1 + z + z^2| = 4
\]
This can be written as:
\[
|1 + x + yi + (x + yi)^2| = 4
\]
Simplifying the expression inside the magnitude, we get:
\[
|1 + x + yi + (x^2 - y^2 + 2xyi)| = 4
\]
\[
|1 + x + x^2 - y^2 + (y + 2xy)i| = 4
\]
The magnitude of a complex number \(a + bi\) 
is given by \(\sqrt{a^2 + b^2}\). So we have:
\[
\sqrt{(1 + x + x^2 - y^2)^2 + (y + 2xy)^2} = 4
\]
Squaring both sides, we get:
\[
(1 + x + x^2 - y^2)^2 + (y + 2xy)^2 = 16
\]
To find the maximum value of \(y\), we can use calculus to find
the critical points of this equation. 
However, a more straightforward approach is to recognize
that the maximum value of \(y\) will occur 
when the imaginary part of \(z\) is maximized.

By analyzing the equation, we can see that the maximum value of \(y\)
occurs when the real part of the expression inside the magnitude
is minimized and the imaginary part is maximized.
This happens when \(x = -\frac{1}{2}\), which simplifies the equation to:
\[
\left(1 - \frac{1}{2} + \left(-\frac{1}{2}\right)^2 
- y^2\right)^2 + (y + 2\left(-\frac{1}{2}\right)y)^2 = 16
\]
Simplifying further, we get:
\[
\left(1 - \frac{1}{2} + \frac{1}{4} - y^2\right)^2 + (y - y)^2 = 16
\]
\[
\left(\frac{3}{4} - y^2\right)^2 = 16
\]
Taking the square root of both sides, we get:
\[
\left|\frac{3}{4} - y^2\right| = 4
\]
This gives us two cases:
1. \(\frac{3}{4} - y^2 = 4\)
2. \(\frac{3}{4} - y^2 = -4\)

Solving the first case:
\[
\frac{3}{4} - y^2 = 4
\]
\[
y^2 = \frac{3}{4} - 4 = -\frac{13}{4}
\]
This is not a valid solution since \(y^2\) cannot be negative.

Solving the second case:
\[
\frac{3}{4} - y^2 = -4
\]
\[
y^2 = \frac{3}{4} + 4 = \frac{19}{4}
\]
So the maximum value of \(y\) is:
\[
y = \sqrt{\frac{19}{4}} = \frac{\sqrt{19}}{2}
\]
Thus, the maximum value of the imaginary part of \(z\) 
is \(\frac{\sqrt{19}}{2}\). 
The values \(m\) and \(n\) are 19 and 2, respectively,
and they are relatively prime. Therefore, \(m + n = 19 + 2 = 21\).

The final answer is:
\[
\boxed{21}
\]
\end{lstlisting}

\textbf{GRPO-N:} It exhibits a reduction in erroneous code generation relative to GRPO.
\begin{lstlisting}
To solve the problem, we need to find the maximum imaginary part of a 
complex number \(z\) that satisfies \(|1 + z + z^2| = 4\). Let's break
down the problem into manageable steps:

1. Express \(z\) in terms of its real and imaginary parts:
Let \(z = x + yi\), where \(x\) and \(y\) are real numbers.
2. Substitute \(z\) into the given equation: 
The equation becomes \(|1 + (x + yi) + (x + yi)^2| = 4\). 
Simplify the expression inside the absolute value:
   \[
   1 + (x + yi) + (x + yi)^2 = 1 + x + yi + x^2 - y^2 + 2xyi 
   = (1 + x + x^2 - y^2) + (y + 2xy)i
   \]
   The magnitude of this complex number is:
   \[
   \sqrt{(1 + x + x^2 - y^2)^2 + (y + 2xy)^2} = 4
   \]
3. Square both sides to remove the square root: 
This gives us the equation:
   \[
   (1 + x + x^2 - y^2)^2 + (y + 2xy)^2 = 16
   \]
4. Find the maximum value of \(y\):
We need to maximize \(y\) subject to the above equation.
We can use calculus or numerical methods to find the maximum value of \(y\),
but for now, let's proceed with a numerical approach 
using Python to find the maximum value of \(y\).

We'll use a grid search method to find the maximum value of \(y\).
We'll iterate over a range of \(x\) and \(y\) values and 
heck which values satisfy the equation. 
We'll keep track of the maximum \(y\) value found.

Let's write the Python code to perform this search.
```python
import numpy as np

# Define the function that represents the equation
def equation(x, y):
    return (1 + x + x**2 - y**2)**2 + (y + 2*x*y)**2 - 16

# Initialize variables to keep track of the maximum y value
max_y = -np.inf
max_x = 0

# Define the range and step size for x and y
x_vals = np.linspace(-5, 5, 500)
y_vals = np.linspace(-5, 5, 500)

# Iterate over the range of x and y values
for x in x_vals:
    for y in y_vals:
        if np.isclose(equation(x, y), 0, atol=1e-5):
            if y > max_y:
                max_y = y
                max_x = x

print((max_y, max_x))
```
```output
(-0.030060120240481286, -1.997995991983968)
```
The numerical method suggests that the maximum value of \(y\) 
is approximately \(-0.03\), which is very close to 0.
This indicates that the maximum imaginary part of \(z\) is likely 0,
but let's verify this by checking the exact form of the solution.

To find the exact maximum imaginary part, we can use 
the fact that the equation \(|1 + z + z^2| = 4\) represents 
a curve in the complex plane. 
We can analyze this curve to find the maximum imaginary part analytically.

Let's consider the equation again:
\[
(1 + x + x^2 - y^2)^2 + (y + 2xy)^2 = 16
\]
To maximize \(y\), we can look at the critical points of this equation. 
We can use the method of Lagrange multipliers 
or directly analyze the equation to find the maximum value of \(y\).

By analyzing the equation, we can see that 
the maximum imaginary part occurs when \(x = -1/2\).
Substituting \(x = -1/2\) into the equation, we get:
\[
(1 - \frac{1}{2} + \frac{1}{4} - y^2)^2 + (y - y)^2 
= 16 \implies \left(\frac{3}{4} - y^2\right)^2 = 16
\]
Solving for \(y\), we get:
\[
\frac{3}{4} - y^2 = \pm 4 \implies y^2 = \frac{3}{4} \pm 4
\]
Since \(y^2\) must be non-negative, we take the positive root:
\[
y^2 = \frac{3}{4} + 4 = \frac{19}{4} \implies y = \pm \frac{\sqrt{19}}{2}
\]
The maximum imaginary part is \(\frac{\sqrt{19}}{2}\). 
Therefore, \(m = 19\) and \(n = 2\), so \(m + n = 21\).

Thus, the final answer is:
\[
\boxed{21}
\]
\end{lstlisting}

\end{document}